\documentclass{article}
\usepackage{graphicx}
\usepackage{subcaption}
\usepackage{booktabs}
\usepackage{multirow}
\usepackage{xltabular}
\usepackage{longtable}
\usepackage[accsupp]{axessibility}
\usepackage[ruled,vlined,linesnumbered]{algorithm2e}
\usepackage{hyperref}
\usepackage{orcidlink}
\usepackage{microtype}
\usepackage{amsmath}
\usepackage{amssymb}
\usepackage{mathtools}
\usepackage{amsthm}
\usepackage[capitalize,noabbrev]{cleveref}
\usepackage[textsize=tiny]{todonotes}

\usepackage[accepted]{icml2026}
\usepackage[textsize=tiny]{todonotes}

\icmltitlerunning{KG-FairDiff: Knowledge Graph-Guided Prompt Refinement for Fair TTI}

\theoremstyle{plain}
\newtheorem{theorem}{Theorem}[section]
\newtheorem{proposition}[theorem]{Proposition}

\theoremstyle{definition}

\theoremstyle{remark}

\begin{document}

\twocolumn[
  \icmltitle{KG-FairDiff: Knowledge Graph-Guided Prompt Refinement for Demographically Fair Text-to-Image Generation}
` \icmlsetsymbol{equal}{*}
  \begin{icmlauthorlist}
    \icmlauthor{Farbod Davoodi}{equal,qcri}
    \icmlauthor{Seyed Reza Tavakoli Shiyadeh}{equal,qcri}
    \icmlauthor{Pooria Safaei}{qcri}
    \icmlauthor{Sana Harighi}{milan}
    \icmlauthor{Parsa Gholami}{sharifce}
    \icmlauthor{Amirali Amini}{sharifce}
    \icmlauthor{Kimia Vanaei}{qcri}
    \icmlauthor{Emad Firoozi}{sharifee}
    \icmlauthor{Parham Abed Azad}{sharifce}
    \icmlauthor{Babak Khalaj}{sharifee}
    \icmlauthor{Siavash Ahmadi}{eri}
    \icmlauthor{Amir Hossein Payberah}{kth}
    \icmlauthor{Mohammad Hossein Rohban}{sharifce}
    \icmlauthor{Soheil Kolouri}{vanderbilt}
    \icmlauthor{Ali Diba}{qcri}
  \end{icmlauthorlist}

  \icmlaffiliation{sharifce}{Department of Computer Engineering, Sharif University of Technology, Tehran, Iran}
  \icmlaffiliation{sharifee}{Department of Electrical Engineering, Sharif University of Technology, Tehran, Iran}
  \icmlaffiliation{eri}{Electronics Research Institute, Sharif University of Technology, Tehran, Iran}
  \icmlaffiliation{qcri}{Qatar Computing Research Institute, Hamad Bin Khalifa University, Doha, Qatar}
  \icmlaffiliation{milan}{Department of Electronics, Information and Bioengineering, Politecnico di Milano, Milan, Italy}
  \icmlaffiliation{kth}{Department of Computing and Learning Systems (CLS) at KTH Royal Institute of Technology}
  \icmlaffiliation{vanderbilt}{Department of Computer Science, College of Connected Computing, Vanderbilt University Department of Electrical and Computer Engineering, Vanderbilt University}

  \icmlcorrespondingauthor{Ali Diba}{adiba@hbku.edu.qa}
  \icmlcorrespondingauthor{Siavash Ahmadi}{s.ahmadi@sharif.edu}
  \icmlkeywords{Fairness, Text-to-Image Generation, Knowledge Graphs, Prompt Engineering}

  \vskip 0.3in
]

\printAffiliationsAndNotice{\icmlEqualContribution}


\begin{abstract}
Text-to-Image (TTI) systems are now everyday infrastructure for journalism, education, advertising, and public communication, and the demographic and cultural stereotypes they inherit from training data---rendering women, people of colour, older adults, and non-Western cultures as under-represented or caricatured---become a population-level harm at deployment scale. Existing mitigations either require costly retraining, infeasible for the closed-source backbones that dominate consumer products, or rely on fixed demographic templates that ignore cultural context.ixed demographic templates that ignore cultural context. We present \textbf{KG-FairDiff}, a model-agnostic, inference-time framework that formalises fairness-aware prompt refinement as a constrained optimisation problem and operationalises it as a closed-loop pipeline: a knowledge graph of $\sim$1{,}200 culture- and bias-related triples retrieves structured context, an LLM rewriter proposes refinements, and a validator accepts only prompts that reduce a divergence-based fairness loss while preserving semantic fidelity to the user's original intent. We prove a finite-termination bound for the refinement loop, contribute a mathematically consistent evaluation suite linking Bias-P/Bias-W to divergence from target distributions and ENS to KL divergence, and audit eight widely-deployed backbone generators. KG-FairDiff substantially reduces gender, race, age, and intersectional disparities while preserving prompt semantics, offering a practical, deployment-ready route to more equitable generative AI.
\end{abstract}

\section{Introduction}

Text-to-Image (TTI) systems have moved from research demonstrations to everyday visual-production infrastructure---illustrating news articles, generating stock imagery, accompanying classroom materials, and answering ``show me a\dots'' queries for a global audience~\cite{Miao_2024_CVPR,shin2024,anonymous2024diverse}. As deployment scales, so does the social weight of these systems' default behaviours, and a growing body of audits shows that those defaults are systematically skewed~\cite{luccioni2023stable,vice2025exploringbias100texttoimage,wan2024surveybiastexttoimagegeneration}: CEOs are rendered as White men, scientists as young East Asian men, nurses as women, and many cultural traditions are erased or caricatured in favour of Western-centric visual priors~\cite{karkkainenfairface,wang-etal-2023-t2iat}. The harm is not the failure of any single image but the \emph{population-scale} repetition of narrow representations across large numbers of generations.

Mitigating this harm sits at the intersection of trustworthy AI and AI for social good: it requires auditing what deployed models produce, intervening under realistic deployment constraints, and demonstrating measurable benefit without new failure modes. Existing approaches fall short on at least one of these axes. Dataset rebalancing and latent-space regularisation~\cite{kim2025rethinkingtrainingdebiasingtexttoimage,li2024fair,esposito2023mitigatingstereotypicalbiasestext} require access to training data or model internals, which is infeasible for closed-source generators. Lighter-weight prompt-level methods such as MinorityPrompt~\cite{um2024minorityprompt} and PreciseDebias~\cite{precisedebias2024} typically rely on fixed demographic templates and ignore structured cultural knowledge, leaving images culturally tokenistic rather than specific. Cultural benchmarks~\cite{shi-etal-2024-culturebank,liu2025culturevlmcharacterizingimprovingcultural} document this gap but do not provide a closed-loop intervention, and related work on retrieval-grounded prompt optimisation~\cite{cui2025surveyautomaticpromptoptimization,su2024dragindynamicretrievalaugmented} has not yet been brought to bear on fairness-aware TTI generation.

We propose \textbf{KG-FairDiff}, a knowledge-graph-guided prompt refinement framework designed for realistic TTI deployment. It runs entirely at inference time and treats the generator as a black box, so it applies to both proprietary and open-weight backbones. A triple-embedded knowledge graph of $\sim$1{,}200 culture- and bias-related triples retrieves structured context for a prompt; an LLM rewriter proposes refinements; and a closed-loop validator accepts only those that reduce a divergence-based fairness loss while keeping cosine similarity to the user's original prompt above a fidelity threshold. The combination of structured cultural grounding and an explicit acceptance rule improves representation without inducing semantic drift or superficial demographic insertion. The paper contributions are as follows: a fidelity threshold. The combination of structured cultural grounding and an explicit acceptance rule improves representation without inducing semantic drift or superficial demographic insertion. The paper contributions are as follows:
\begin{enumerate}
  \item \textbf{Formulation and algorithm.} We formalise fairness-aware prompt refinement as a constrained optimisation problem with explicit fairness, semantic-fidelity, and style-preservation terms, and give an iterative refinement algorithm with a finite termination bound under a monotone-improvement assumption on the validator.
  \item \textbf{A reusable knowledge graph for fairness-aware prompting.} We construct and release a hand-curated KG of ${\sim}1{,}200$ triples spanning stereotype, counter-stereotype, and cultural grounding relations, sourced from the McGillNLP Bias Dataset and CultureBank and double-annotated for factual accuracy and cultural sensitivity.
  \item \textbf{An evaluation suite with theoretical grounding.} We provide a mathematically consistent evaluation framework that links Bias-P/Bias-W to divergence from target distributions, relates ENS to KL divergence, and uses Wasserstein / Gromov--Wasserstein alignment together with CLIP directional similarity to assess semantic preservation supporting the evidence standards needed to credibly claim social benefit.
  \item \textbf{An audit of eight deployed generators against three prompt-level baselines.} On a 100-prompt benchmark spanning 50 occupations from the U.S.~BLS taxonomy, KG-FairDiff improves fairness across gender, race, age, and intersectional axes on eight backbones (GPT-Image-1, Qwen-VL-2512, SD~3.5 Large, Zimage, SD~Lightning, Gemini~3~Pro, Gemini~2.5~Flash, GPT-5-image) and outperforms MinorityPrompt, PreciseDebias, and FairImagen e.g., up to $20\times$ larger Bias-W reductions and $13\times$ larger ENS gains over MinorityPrompt.
  \item \textbf{Isolation of the KG component and honest failure analysis.} An ablation across (embedding model, LLM rewriter) pairs shows that KG retrieval contributes independently of the rewriter, and we transparently report regressions on SD~3.5~Large and Gemini~3~Pro, attributing them to safety-filter interactions and conflict between intersectional cultural cues.
\end{enumerate}

\section{Related Work}

Systematic stereotyping in TTI is characterised by Masrouri et al.~\cite{masrourisaadat2024analyzingqualitybiasperformance}, Miao et al.~\cite{Miao_2024_CVPR}, and Gallegos et al.~\cite{gallegos2024}. Wan et al.~\cite{wan2024surveybiastexttoimagegeneration} present a survey classifying bias sources. Vice et al.~\cite{vice2025exploringbias100texttoimage} audit over 100 TTI models. Wang et al.~\cite{wang-etal-2023-t2iat} propose T2IAT for bias measurement. D'Incà et al.~\cite{dinca2024openbias} introduce OpenBias. Lee et al.~\cite{lee2024heim} propose HEIM, a holistic benchmark. Weng et al.~\cite{weng2025imagesspeaklouderwords} analyse TTI bias from a causal mediation perspective. Prerak~\cite{prerak2024} surveys bias mitigation methods across multiple model families.

\noindent\textbf{TTI-specific intervention strategies} include adversarial fine-tuning~\cite{esposito2023mitigatingstereotypicalbiasestext}, data filtering~\cite{kim2025rethinkingtrainingdebiasingtexttoimage}, fairness-aware conditioning~\cite{li2024fair}, and self-contrastive fine-tuning~\cite{liu2024scoftselfcontrastivefinetuningequitable}. Friedrich et al.~\cite{friedrich2024fairdiffusion} propose Fair Diffusion. Bonna et al.~\cite{bonna2025debiaspi} introduce DebiasPI. Sahili et al.~\cite{sahili2025faircot} propose FairCoT, which uses chain-of-thought reasoning to improve fairness. Zhang et al.~\cite{zhang2023itigeniccv} present ITI-GEN. Dai et al.~\cite{dai202415mmultimodalfacialimagetext} provide large-scale multimodal face datasets relevant to demographic fairness evaluation.

\noindent\textbf{Cultural grounding} is recognised as a distinct fairness axis in TTI, separate from the broader LLM bias literature~\cite{bang2024measuringpoliticalbiaslarge,gallegos2024,lin2024investigatingbiasllmbasedbias,fan2025biasguardreasoningenhancedbiasdetection}. CultureVLM~\cite{liu2025culturevlmcharacterizingimprovingcultural}, Culture-TRIP~\cite{jeong2025culturetripculturallyawaretexttoimagegeneration}, CultureBank~\cite{shi-etal-2024-culturebank}, CDEval~\cite{wang-etal-2024-cdeval}, CIVICS~\cite{pistilli2024civicsbuildingdatasetexamining}, and Naous et al.~\cite{naous-etal-2024-beer} provide relevant benchmarks and analyses. Vasilev et al.~\cite{Vasilev_2024} demonstrate culture-specific T2I dataset adaptation.

\section{Method}

KG-FairDiff places fairness intervention entirely at the prompt level, leaving the underlying generator untouched and therefore deployable on any black-box TTI backbone. The pipeline (Figure~\ref{fig:ttibias}) has three stages: (i)~we build a knowledge graph encoding both stereotypical associations to be countered and culturally grounded concepts to be preserved; (ii)~at inference time, a prompt is iteratively rewritten under guidance from KG-retrieved triples and accepted only if a closed-loop validator confirms both a fairness improvement and sufficient semantic fidelity to the original; (iii)~the accepted prompt is then passed to the unmodified generator. The remainder of this section formalises the resulting constrained optimization problem, specifies the knowledge graph and rewrite--validate loop, and proves a finite-termination bound for the refinement procedure.

\subsection{Problem Formulation}
Let $p \in \mathcal{P}$ be a text prompt and let a TTI model induce $x \sim G(\cdot \mid p)$, where $G$ denotes a single backbone generator from the set $\mathcal{G}$ of all backbone generators. Let $h_G : \mathcal{X} \to \mathcal{A}_G$ be an attribute classifier with $n_G = |\mathcal{A}_G|$ categories. The prompt-induced attribute distribution is $\pi_p(a) = \mathbf{P}_{x \sim G(\cdot \mid p)}[h_G(x)=a]$, and $\mathcal{P}$ denotes the space of all valid text prompts with $p_0 \in \mathcal{P}$ the initial (unrefined) prompt.

\noindent\textbf{Fairness loss.} For neutral prompts we use a uniform target $u_G(a)=1/n_G$:
\begin{equation}
  \mathcal{L}_{\mathrm{fair}}(p;G) = D(\pi_p \| u_G(\cdot)),
  \label{eq:fair_loss}
\end{equation}
where $D(\cdot\|\cdot)$ denotes KL divergence.

\noindent\textbf{Semantic fidelity.} Let $e(\cdot)$ be a text embedding function. Given an initial prompt $p_0 \in \mathcal{P}$, we require:
\begin{equation}
  \mathrm{sim}(p,p_0) = \frac{\langle e(p),e(p_0)\rangle}{\|e(p)\|_2\|e(p_0)\|_2} \geq \tau.
  \label{eq:sim}
\end{equation}

\noindent\textbf{Constrained optimisation.} Let $\mathcal{G}$ denote the set of backbone generators, $w_G \geq 0$ their associated weights, and $\mathcal{L}_{\mathrm{style}}(p)$ a style-preservation regulariser. Given $p_0$, we seek:

\begin{equation}
\begin{aligned}
p^* \in \arg\min_{p\in\mathcal{P}} \quad &
\sum_{G\in\mathcal{G}} w_G\,\mathcal{L}_{\mathrm{fair}}(p;G)
+ \lambda\,\mathcal{L}_{\mathrm{style}}(p) \\
\text{s.t.} \quad &
\mathrm{sim}(p,p_0)\geq\tau.
\end{aligned}
\label{eq:opt}
\end{equation}

where $\mathcal{L}_{\mathrm{style}}(p)$ penalises stylistic drift (e.g., over-specification of demographics), and $\lambda \geq 0$ is a trade-off weight.

\noindent\textbf{Handling demographic-restricted prompts.}
When the initial prompt $p_0$ already contains explicit demographic attributes (e.g., ``a female nurse''), our framework treats those attributes as hard constraints and restricts the search space $\mathcal{P}$ accordingly: the rewriter is instructed to preserve all explicit demographic mentions while introducing diversity only along the remaining, unconstrained attribute axes. This prevents the framework from overriding user-specified demographics. The example in Fig.~\ref{fig:ttibias} involves an attribute-neutral prompt (no gender or age specified), and the introduction of ``male'' and ``middle-aged'' in the refined prompt represents one possible diverse completion rather than a prescribed demographic; we acknowledge this as a limitation and discuss it further in Section~\ref{sec:limitations}.

\subsection{Knowledge Graph and Retrieval}

We represent domain knowledge as a directed labelled multigraph $\mathcal{K}=(\mathcal{V},\mathcal{R},\mathcal{E})$. Triples encode stereotypical associations (to be countered) and culturally grounded concepts (to be preserved).

\noindent\textbf{KG Construction.} The KG comprises approximately 1,200 hand-curated triples spanning three categories: (i)~\emph{stereotype triples} drawn from the McGillNLP Bias Dataset~\cite{gallegos2024}, encoding known occupational and demographic stereotypes; (ii)~\emph{counter-stereotype triples} asserting the inverse associations, manually authored to directly negate each stereotype triple; and (iii)~\emph{cultural grounding triples} sourced from CultureBank~\cite{shi-etal-2024-culturebank}, encoding non-Western cultural concepts, attire, and practices. Triples follow the schema g non-Western cultural concepts, attire, and practices. Triples follow the schema \texttt{(subject, relation, object)} with relation types \texttt{stereotyped\_as}, \texttt{counter\_stereotype}, and \texttt{associated\_culture}. Each triple was reviewed by at least two annotators for factual accuracy and cultural sensitivity. Retrieval returns top-$k$ triples by cosine similarity: $\mathrm{Retrieve}_k(p) = \mathrm{TopK}_{t\in\mathcal{E}}\cos(v_p,\,v_t)$.

\begin{figure*}[t]
  \centering
  \includegraphics[width=\textwidth]{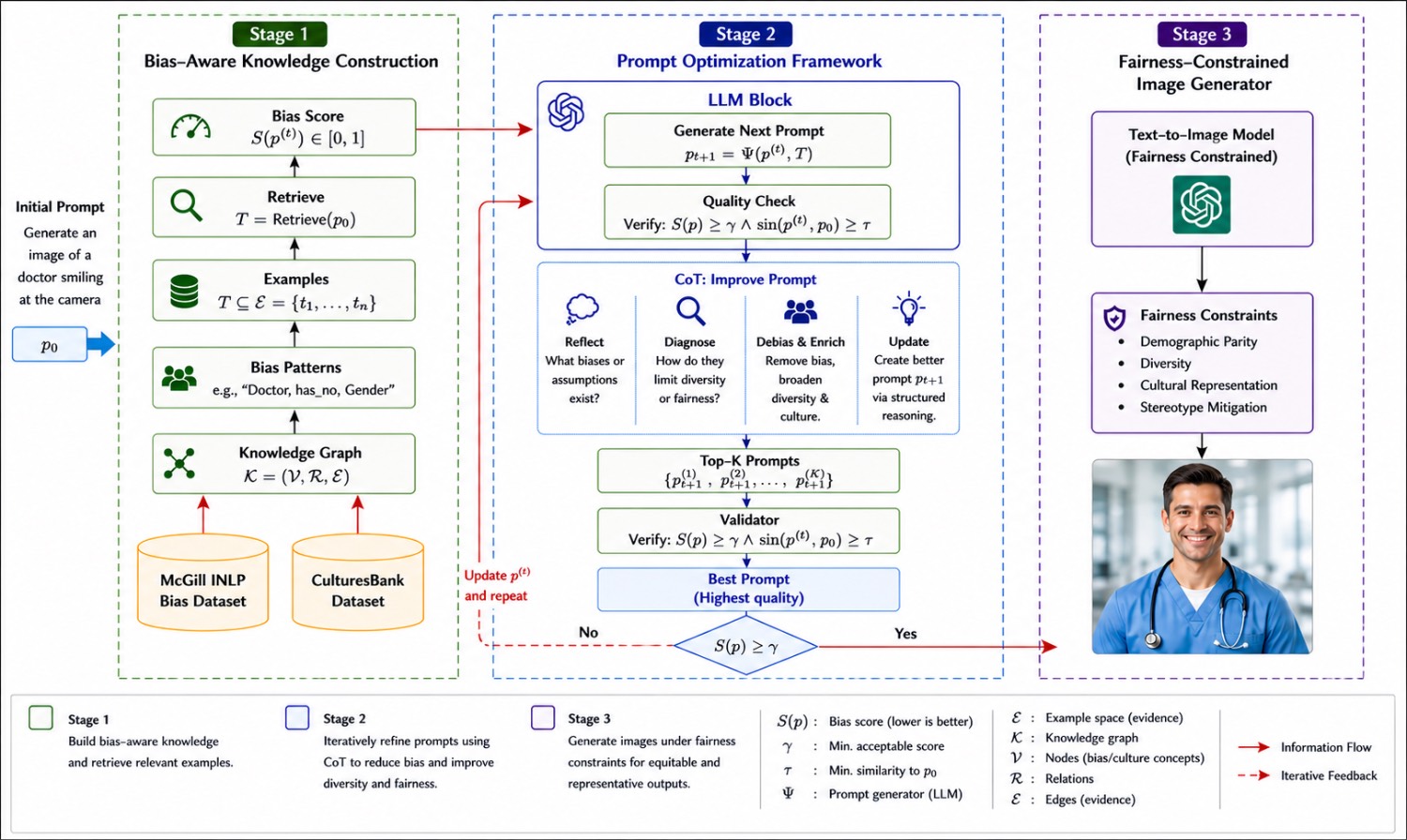}
  \caption{\textbf{KG-FairDiff pipeline.} Stage~1: knowledge graph construction. Stage~2: iterative prompt optimisation. Stage~3: fairness-constrained TTI generation.}
  \label{fig:ttibias}
\end{figure*}

\subsection{Prompt Rewrite and Closed-Loop Validation}

Given retrieved triples $T=\mathrm{Retrieve}_k(p)$, a rewrite operator produces $p' = \Psi(p, T)$. A validator assigns score $S(p)\in[0,1]$ reflecting bias intensity, representational inclusivity, and cultural appropriateness. A prompt is accepted if $S(p)\geq\gamma$ and $\mathrm{sim}(p,p_0)\geq\tau$. Algorithm~\ref{algo:algo} gives the full procedure.

We acknowledge a potential circularity: using GPT-4o both as the rewriter $\Psi$ and the validator $S$ means the model evaluates its own outputs. In practice, the rewrite and validation calls use distinct system prompts with complementary objectives (generation vs.\ critique), and the downstream quantitative fairness metrics (Bias-W, KL, ENS) provide external validation. Nevertheless, we recognise this as a limitation and discuss alternatives such as using a held-out LLM or human evaluation in Section \ref{sec:limitations}.

\begin{algorithm}[t]
\caption{\textbf{Iterative Prompt Refinement}}
\label{algo:algo}
\KwIn{Prompt $p_0$, KG $\mathcal{K}$, thresholds $\gamma,\tau$, max iters $T_{\max}$}
\KwOut{Refined prompt $p^*$}
$p^{(0)} \gets p_0$\;
\For{$t = 0, \dots, T_{\max}-1$}{
  $T^{(t)} \gets \mathrm{Retrieve}_k(p^{(t)})$\;
  $\tilde{p}^{(t)} \gets \Psi(p^{(t)}, T^{(t)})$\;
  \If{$S(\tilde{p}^{(t)}) \geq \gamma$ \textbf{and} $\mathrm{sim}(\tilde{p}^{(t)}, p_0) \geq \tau$}{
    \Return $\tilde{p}^{(t)}$\;
  }
  $p^{(t+1)} \gets \tilde{p}^{(t)}$\;
}
\Return $p^{(T_{\max})}$\;
\end{algorithm}

\subsection{Theoretical Analysis}

Under a monotone-improvement assumption on the validator, the loop terminates in at most $T \leq \lceil(\gamma - S(p^{(0)}))/\epsilon\rceil$ iterations. Minimising KL divergence to uniform is equivalent to maximising Shannon entropy and ENS. By convexity of the $\ell_2$ norm, dataset-level bias (Bias-W) cannot exceed average per-prompt bias (Bias-P). Formal proofs are provided in the Appendix.

\subsection{Stopping Criteria}
\label{app:stopping}

We derive a worst-case termination guarantee under a sufficient condition on the validator. Assume the validator score increases by at least $\epsilon>0$ whenever it is below threshold---a condition that holds if the LLM rewriter makes measurable progress on each iteration:
\begin{equation}
  S(p^{(t+1)}) \geq \min\!\left\{S(p^{(t)})+\epsilon,\,1\right\}
  \quad\text{whenever }S(p^{(t)})<\gamma.
\end{equation}
Under this assumption the loop terminates in at most
\begin{equation}
  T \;\leq\; \left\lceil\frac{\gamma - S(p^{(0)})}{\epsilon}\right\rceil
\end{equation}
iterations or at $T_{\max}$.

\noindent\textit{Remark.} The monotone-improvement assumption is sufficient but not necessary: the algorithm still terminates at $T_{\max}$ without it. The 
$S$ need not be monotone in practice (race KL spikes at Iteration~3), confirming that the bound is conservative rather than tight. Empirical correlation between $S(p)$ and downstream fairness metrics (Bias-W, KL, ENS) across prompts and iterations is left for future work; establishing this connection rigorously would strengthen the theoretical grounding of the validator as a proxy for $\mathcal{L}_{\mathrm{fair}}$.

\begin{figure*}[tp]
\centering

\begin{subfigure}[t]{0.49\textwidth}
    \centering
    \includegraphics[width=0.65\linewidth]{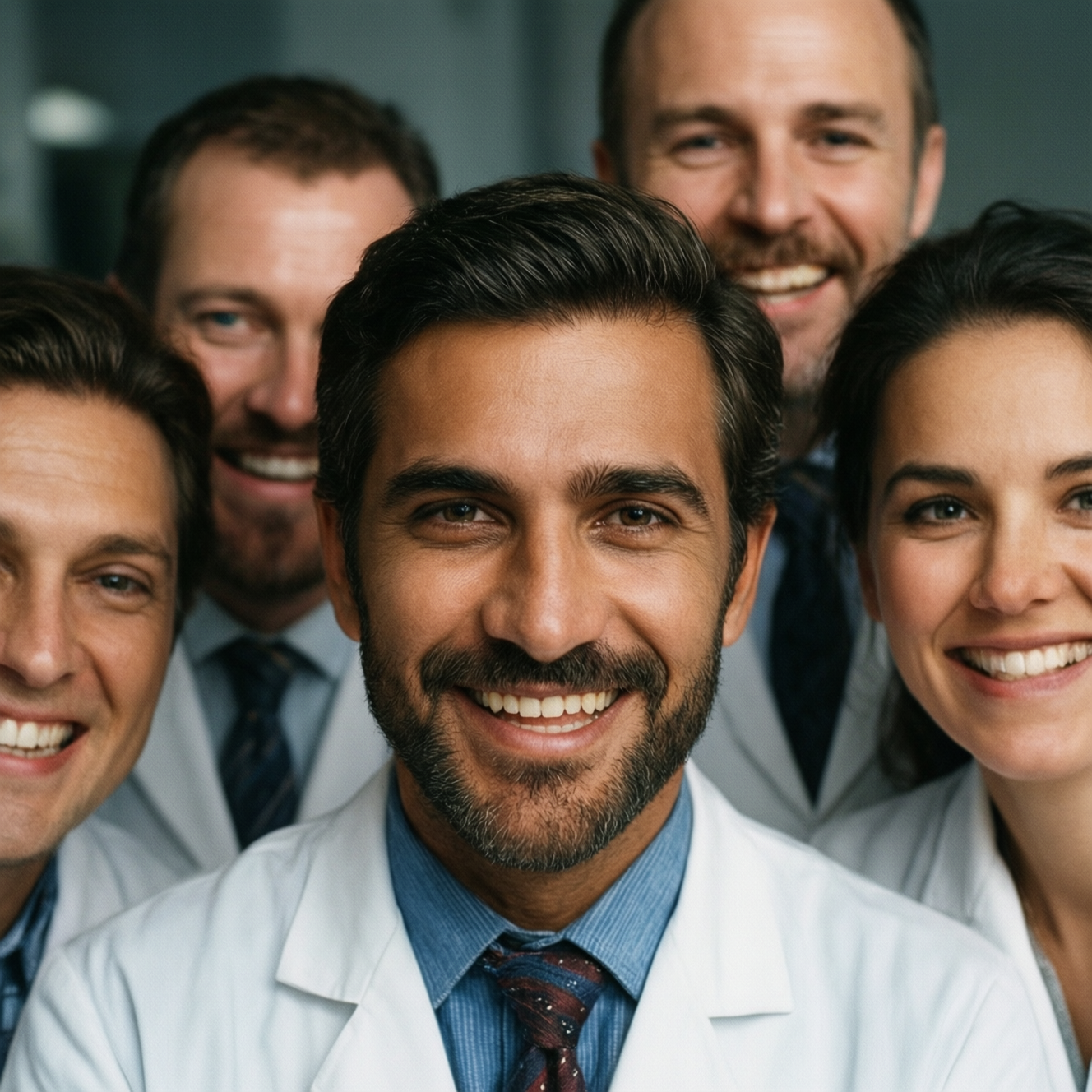}
\end{subfigure}
\hfill
\begin{subfigure}[t]{0.49\textwidth}
    \centering
    \includegraphics[width=0.65\linewidth]{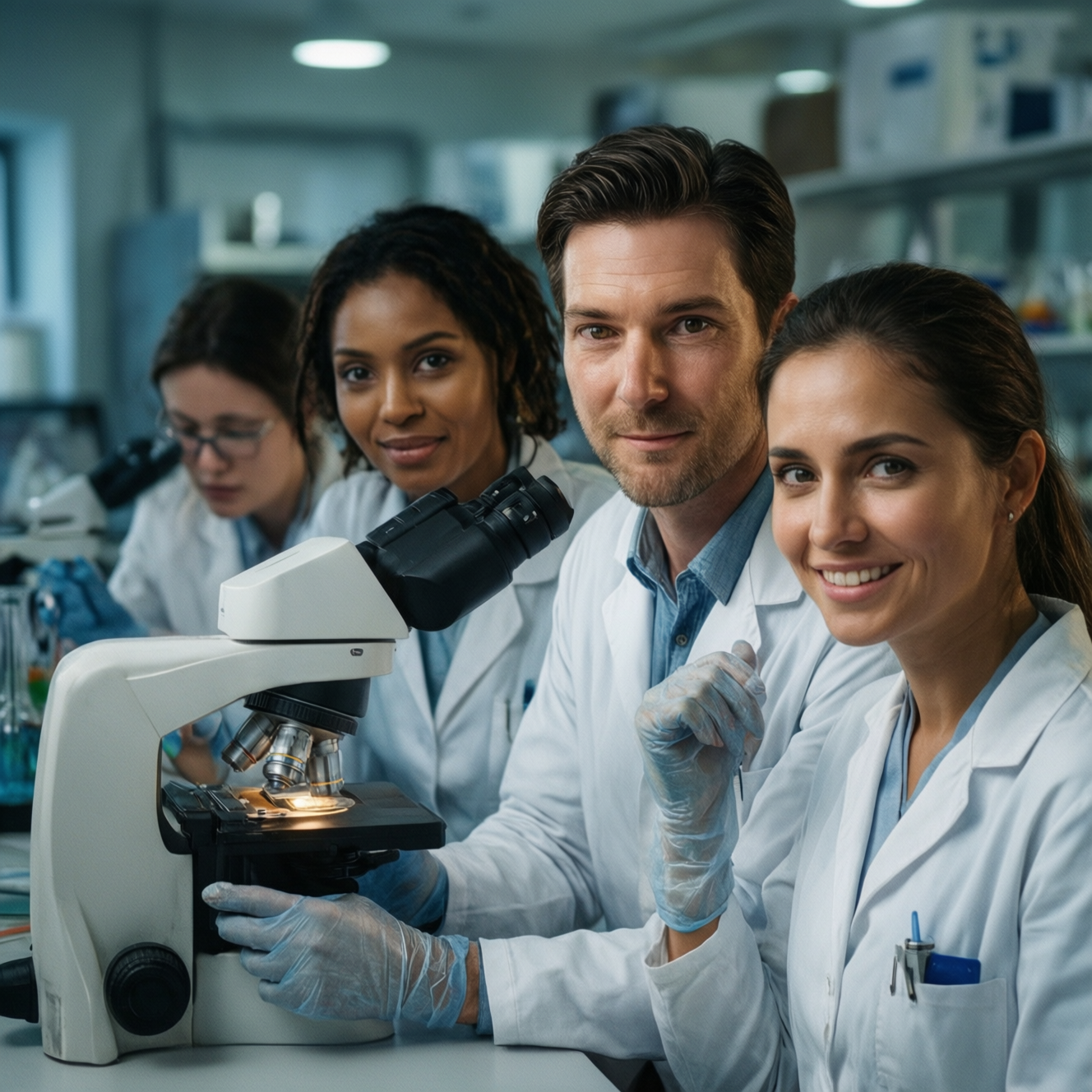}
\end{subfigure}

\caption{Qualitative comparison for Prompt Comparison. Left: an image generated with the generic prompt. Right: an image generated with the enhanced prompt.}
\label{fig:prompt_comparison_4}
\end{figure*}

\section{Experiments}

\subsection{Setup}

\noindent\textbf{Models and data.} We construct a benchmark of 100 occupation-focused prompts covering 50 distinct professions drawn from the U.S.\ Bureau of Labor Statistics occupational taxonomy, selected to span a wide range of gender- and race-stereotyped roles (e.g., CEO, nurse, engineer, domestic worker). For each profession, two prompt variants are generated: a minimal variant (``Generate an image of a [profession]'') and a context-rich variant (``Generate a close-up image of a group of [professions] who are smiling at my camera full face''). We produce 10 images per prompt per condition with Stable Diffusion v1.5, SD v3-lightning, GPT-image-1, and four other generators.

\noindent\textbf{Knowledge graph.} The KG comprises approximately 1,200 hand-curated triples spanning stereotype, counter-stereotype, and cultural grounding triples drawn from CultureBank~\cite{shi-etal-2024-culturebank} and the McGillNLP Bias Dataset.

\noindent\textbf{Retrieval and rewriter.} We retrieve $k=5$ nearest triples using text-embedding-3-small. The LLM rewriter and validator are GPT-4o with $\gamma=0.75$, $\tau=0.80$, and $T_{\max}=5$.

\noindent\textbf{Threshold justification.}
The acceptance thresholds $\gamma=0.75$ and $\tau=0.80$ were selected via a grid search on a held-out development set of 20 prompts (not included in the main evaluation). We swept $\gamma \in \{0.60, 0.70, 0.75, 0.80\}$ and $\tau \in \{0.70, 0.75, 0.80, 0.85\}$ and selected the combination that maximised the average improvement in Bias-W while keeping the semantic fidelity loss (measured by $\Delta W_2^2$) below 0.005.

\noindent\textbf{Attribute measurement.} Demographic attributes are estimated using FairFace \cite{karkkainenfairface}. We report Bias-P, Bias-W~\cite{jiang2024mitigating}, ENS, $D_{\mathrm{KL}}$, and ICAD~\cite{li2024fair}.

\subsection{Comparison with Prompt-Level Baselines}
\label{subsec:comparison}

KG-FairDiff substantially outperforms MinorityPrompt~\cite{um2024minorityprompt} and PreciseDebias~\cite{precisedebias2024}. For MinorityPrompt, our method achieves Bias-W reductions up to $20\times$ larger (e.g., $-0.1396$ vs. $-0.0037$ for Race) and ENS gains up to $13\times$ larger. Detailed comparison tables are provided in the Appendix.

\noindent\textbf{MinorityPrompt \cite{um2024minorityprompt}.} optimises prompts to sample low-density regions of the text-conditional distribution in T2I diffusion models. We regenerate images using Stable Diffusion following their experimental protocol. As Table~\ref{tab:minority_comparison} shows, Minority yields only marginal improvements: Bias-W for \textbf{Race} decreases by just $-0.0037$ and \textbf{Gender} bias slightly increases ($+0.0028$). ENS and KL reductions are similarly limited. Table \ref{tab:minority_comparison} quantifies the gap: our method achieves Bias-W reductions up to $20\times$ larger (e.g.\ $-0.1396$ vs.\ $-0.0037$ for Race) and ENS gains up to $13\times$ larger (e.g.\ $+8.79$ vs.\ $+0.68$ for Race--Gender--Age). Although Minority increases detected faces~($+123$), these do not translate into meaningful fairness gains.

\begin{table*}[!ht]
\centering
\caption{%
\textbf{Changes relative to baseline: MinorityPrompt~\cite{um2024minorityprompt} vs.\ KG-FairDiff.} $\Delta = \text{Method} - \text{Baseline}$; negative $\Delta$ for Bias-P, Bias-W, KL indicates improvement; positive $\Delta$ for ENS and ICAD indicates improvement. Both methods are evaluated using Stable Diffusion following MinorityPrompt's experimental protocol on the same 100-prompt set. The Faces Detected row reports the total change in FairFace-detected faces after refinement. Bold values indicate the better result for each metric--attribute pair.}
\label{tab:minority_comparison}
\resizebox{\textwidth}{!}{%
\begin{tabular}{|l|cc|cc|cc|cc|cc|}
\hline
\multirow{2}{*}{Attribute}
  & \multicolumn{2}{c|}{Bias-P $\downarrow$}
  & \multicolumn{2}{c|}{Bias-W $\downarrow$}
  & \multicolumn{2}{c|}{ENS $\uparrow$}
  & \multicolumn{2}{c|}{KL Divergence $\downarrow$}
  & \multicolumn{2}{c|}{ICAD $\uparrow$}\\
\cmidrule(lr){2-3}\cmidrule(lr){4-5}\cmidrule(lr){6-7}\cmidrule(lr){8-9}\cmidrule(lr){10-11}
& $\Delta$ Minority & $\Delta$ KG-FairDiff
& $\Delta$ Minority & $\Delta$ KG-FairDiff
& $\Delta$ Minority & $\Delta$ KG-FairDiff
& $\Delta$ Minority & $\Delta$ KG-FairDiff
& $\Delta$ Minority & $\Delta$ KG-FairDiff \\
\hline

Race 
& $-0.0045$ & $\mathbf{-0.0374}$
& $-0.0037$ & $\mathbf{-0.1396}$
& $+0.03$ & $\mathbf{+1.56}$
& $-0.0276$ & $\mathbf{-0.8229}$
& $\textbf{+0.406}$ & $-0.398$ \\

Gender 
& $+0.0072$ & $\mathbf{-0.0827}$
& $+0.0028$ & $\mathbf{-0.0783}$
& $+0.00$ & $\mathbf{+0.05}$
& $+0.0014$ & $\mathbf{-0.0272}$
& $+0.147$ & $\mathbf{+0.485}$ \\

Age 
& $-0.0090$ & $\mathbf{-0.0305}$
& $-0.0189$ & $\mathbf{-0.0745}$
& $+0.29$ & $\mathbf{+1.55}$
& $-0.1344$ & $\mathbf{-0.5758}$
& $\mathbf{+0.430}$ & $+0.149$ \\

Race--Gender 
& $-0.0006$ & $\mathbf{-0.0275}$
& $-0.0017$ & $\mathbf{-0.0686}$
& $+0.07$ & $\mathbf{+2.75}$
& $-0.0290$ & $\mathbf{-0.7755}$
& $\mathbf{+0.211}$ & $-0.044$ \\

Race--Age 
& $-0.0042$ & $\mathbf{-0.0135}$
& $-0.0074$ & $\mathbf{-0.0494}$
& $+0.42$ & $\mathbf{+5.73}$
& $-0.1588$ & $\mathbf{-1.2131}$
& $\mathbf{+0.373}$ & $-0.191$ \\

Gender--Age 
& $-0.0037$ & $\mathbf{-0.0163}$
& $-0.0097$ & $\mathbf{-0.0303}$
& $+0.45$ & $\mathbf{+2.40}$
& $-0.1170$ & $\mathbf{-0.5042}$
& $\mathbf{+0.238}$ & $-0.425$ \\

Race--Gender--Age 
& $-0.0020$ & $\mathbf{-0.0077}$
& $-0.0037$ & $\mathbf{-0.0232}$
& $+0.68$ & $\mathbf{+8.79}$
& $-0.1443$ & $\mathbf{-1.0997}$
& $\mathbf{+0.230}$ & $-0.607$ \\

\hline
Faces Detected
  & \multicolumn{2}{c|}{$+123$ (Minority)}
  & \multicolumn{2}{c|}{$-395$ (KG-FairDiff)}
  & \multicolumn{2}{c|}{--}
  & \multicolumn{2}{c|}{--}
  & \multicolumn{2}{c|}{--} \\
\hline
\end{tabular}}
\end{table*}

\noindent\textbf{PreciseDebias \cite{precisedebias2024}.}
We apply their automatic prompt engineering approach with Stable Diffusion following their protocol. Table~\ref{tab:delta_comparison_merged} reports the changes relative to baseline.

\begin{table*}[!ht]
\centering
\caption{%
\textbf{Changes relative to baseline: PreciseDebias~\cite{precisedebias2024} vs.\ KG-FairDiff.} $\Delta = \text{Method} - \text{Baseline}$; negative $\Delta$ for Bias-P, Bias-W, KL indicates improvement; positive $\Delta$ for ENS indicates improvement. Both methods are applied with Stable Diffusion following PreciseDebias's automatic prompt engineering protocol. Bold values indicate the better result for each metric--attribute pair.}
\label{tab:delta_comparison_merged}
\resizebox{\textwidth}{!}{%
\begin{tabular}{|l|cc|cc|cc|cc|}
\hline
\multirow{2}{*}{Attribute}
  & \multicolumn{2}{c|}{Bias-P $\downarrow$}
  & \multicolumn{2}{c|}{Bias-W $\downarrow$}
  & \multicolumn{2}{c|}{ENS $\uparrow$}
  & \multicolumn{2}{c|}{KL Divergence $\downarrow$} \\
\cmidrule(lr){2-3}\cmidrule(lr){4-5}\cmidrule(lr){6-7}\cmidrule(lr){8-9}
& $\Delta$ Precise & $\Delta$ KG-FairDiff
& $\Delta$ Precise & $\Delta$ KG-FairDiff
& $\Delta$ Precise & $\Delta$ KG-FairDiff
& $\Delta$ Precise & $\Delta$ KG-FairDiff \\
\hline
Race 
& $+0.03$ & $\mathbf{-0.00}$
& $-0.10$ & $\mathbf{-0.13}$
& $\textbf{+1.84}$ & $+1.56$
& $-0.61$ & $\mathbf{-0.82}$ \\

Gender 
& $+0.24$ & $\mathbf{+0.01}$
& $-0.03$ & $\mathbf{-0.07}$
& $+0.06$ & $\textbf{+0.06}$
& $\textbf{-0.03}$ & $-0.02$ \\

Age 
& $+0.04$ & $\mathbf{-0.02}$
& $+0.01$ & $\mathbf{-0.07}$
& $-0.31$ & $\mathbf{+1.55}$
& $+0.10$ & $\mathbf{-0.57}$ \\

Race--Gender 
& $+0.06$ & $\mathbf{-0.00}$
& $-0.05$ & $\mathbf{-0.06}$
& $\textbf{+2.80}$ & $+2.75$
& $-0.57$ & $\mathbf{-0.77}$ \\

Race--Age 
& $+0.02$ & $\mathbf{-0.00}$
& $-0.01$ & $\mathbf{-0.04}$
& $+2.83$ & $\mathbf{+5.73}$
& $-0.45$ & $\mathbf{-1.21}$ \\

Gender--Age 
& $+0.06$ & $\mathbf{-0.00}$
& $+0.00$ & $\mathbf{-0.03}$
& $-0.18$ & $\mathbf{+2.40}$
& $+0.02$ & $\mathbf{-0.50}$ \\

Race--Gender--Age 
& $+0.02$ & $\mathbf{+0.00}$
& $-0.00$ & $\mathbf{-0.02}$
& $+4.14$ & $\mathbf{+8.79}$
& $-0.45$ & $\mathbf{-1.09}$ \\
\hline
\end{tabular}%
}
\end{table*}

\noindent\textbf{FairImagen~\cite{fu2025fairimagenpostprocessingbiasmitigation}.} We use HuggingFace's Stable Diffusion~3 pipeline following their protocol. Table~\ref{tab:fairimagen_vs_kgf} reports the changes relative to baseline.

\begin{table*}[!ht]
\centering
\caption{%
  \textbf{Changes Relative to Baseline: FairImagen\cite{fu2025fairimagenpostprocessingbiasmitigation} vs.\ KG-FairDiff.} Absolute $\Delta$ = (Method) $-$ (Baseline). FairImagen\textsuperscript{$\dagger$} was evaluated using the HuggingFace Stable Diffusion~3 pipeline following the original protocol~\cite{fu2025fairimagenpostprocessingbiasmitigation}; KG-FairDiff was evaluated on SD~v1.5. Since the backbone generators differ, comparisons are indicative only. Dashes (---) indicate attribute groups not evaluated by FairImagen.}
\label{tab:fairimagen_vs_kgf}
\resizebox{\textwidth}{!}{%
\begin{tabular}{|l|cc|cc|cc|cc|cc|}
\toprule
\multirow{2}{*}{\textbf{Attribute Group}}
  & \multicolumn{2}{c|}{\textbf{$\Delta$ Bias-P} $\downarrow$}
  & \multicolumn{2}{c|}{\textbf{$\Delta$ Bias-W} $\downarrow$}
  & \multicolumn{2}{c|}{\textbf{$\Delta$ ENS} $\uparrow$}
  & \multicolumn{2}{c|}{\textbf{$\Delta$ KL} $\downarrow$}
  & \multicolumn{2}{c}{\textbf{$\Delta$ ICAD} $\uparrow$} \\
\cmidrule(lr){2-3}\cmidrule(lr){4-5}\cmidrule(lr){6-7}\cmidrule(lr){8-9}\cmidrule(lr){10-11}
  & FI\textsuperscript{$\dagger$} & \textbf{KGF}
  & FI\textsuperscript{$\dagger$} & \textbf{KGF}
  & FI\textsuperscript{$\dagger$} & \textbf{KGF}
  & FI\textsuperscript{$\dagger$} & \textbf{KGF}
  & FI\textsuperscript{$\dagger$} & \textbf{KGF} \\
\midrule
Race
  & $-0.0003$ & $\mathbf{-0.03}$
  & $-0.0181$ & $\mathbf{-0.1396}$
  & $\textbf{+4.5548}$ & $+1.56$
  & $-0.7867$ & $\mathbf{-0.8229}$
  & $+0.2345$ & $\mathbf{+1.03}$ \\
Gender
  & $-0.0001$ & $\mathbf{-0.11}$
  & $+0.0018$ & $\mathbf{-0.0783}$
  & $-0.3096$ & $\mathbf{+0.05}$
  & $+0.0150$ & $\mathbf{-0.0272}$
  & $\textbf{-0.0948}$ & $-0.23$ \\

Race+Gender
  & $-0.0001$ & $\mathbf{-0.03}$
  & $-0.0025$ & $\mathbf{-0.0686}$
  & $+1.8712$ & $\mathbf{+2.75}$
  & $-0.2231$ & $\mathbf{-0.7755}$
  & $+0.0027$ & $\mathbf{+0.85}$ \\

\bottomrule
\multicolumn{11}{l}{%
  \textsuperscript{$\dagger$}FairImagen evaluated on SD~v3; KG-FairDiff on SD~v1.5; direct comparison is indicative only.} \\
\end{tabular}}
\end{table*}

\subsection{Ablation Study}

\noindent\textbf{Semantic fidelity.} The mean $W_2^2$ changes by $\Delta\bar{W}_2^2 = 0.0011$ on average after refinement, confirming that fairness-aware rewriting preserves cross-modal semantic alignment. More details can be found at Appendix \ref{app:ot_alignment}

\noindent\textbf{Effect of KG-guided retrieval vs.\ plain LLM rewriting.}
A natural question is whether the fairness gains stem from GPT-4o rewriting alone or specifically from KG-guided retrieval. In our LLM/embedding ablation (Table~\ref{tab:llm_emb_ablation}), we observe that the quality of retrieved context (embedding model) interacts significantly with the rewriting quality: switching from \texttt{text-embedding-3-large} to \texttt{qwen-embedding-4b} increases the keyword bias reduction from $\Delta \approx -40$ to $\Delta \approx -77$, a gap that cannot be explained by the LLM rewriter alone since the same GPT-4o model is compared across rows. This strongly suggests that the KG-retrieval step contributes independently to the observed gains.

\begin{table}[t]
\centering
\small
\caption{%
    \textbf{Ablation study on LLM and Embedding Model combinations}
  Each group of four rows corresponds to a different (embedding model, LLM rewriter) pairing.
  \emph{Keyword} scores measure prompt-level lexical bias reduction;
  \emph{LLM} scores measure the LLM validator's bias assessment (lower is better).
  $\Delta = \text{Enhanced} - \text{Original}$; larger negative $\Delta$ indicates greater bias reduction.
  All configurations use $T_{\max}=5$ refinement iterations.}
\label{tab:llm_emb_ablation}
\resizebox{\columnwidth}{!}{%
\begin{tabular}{l|l|ccc|ccc}
\toprule
\textbf{Embedding} & \textbf{LLM}
  & \textbf{KW$_\text{orig}$} & \textbf{KW$_\text{enh}$} & \textbf{KW $\Delta$}
  & \textbf{LLM$_\text{orig}$} & \textbf{LLM$_\text{enh}$} & \textbf{LLM $\Delta$} \\
\midrule
\midrule
\multirow{4}{*}{\texttt{text-embedding-3-large}} & \multirow{4}{*}{\texttt{GPT-4o}} 
    & 13 & 83 & \textbf{70} & 35 & 90 & \textbf{55} \\
    & & 13 & 90 & \textbf{77} & 35 & 90 & \textbf{55} \\
    & & 13 & 83 & \textbf{70} & 30 & 90 & \textbf{60} \\
    & & 13 & 100 & \textbf{87} & 35 & 92 & \textbf{57} \\
\midrule
\multirow{4}{*}{\texttt{qwen-embedding-4b}} & \multirow{4}{*}{\texttt{qwen-3.5-27b}}
  & 13 & 90 & \textbf{77} & 15 & 88 & \textbf{73} \\
& & 13 & 71 & \textbf{58} & 11 & 89 & \textbf{78} \\
& & 13 & 83 & \textbf{70} & 13 & 95 & \textbf{82} \\
& & 13 & 100 & \textbf{87} & 14 & 94 & \textbf{80} \\
\midrule
\multirow{4}{*}{\texttt{qwen-embedding-8b}} & \multirow{4}{*}{\texttt{qwen-3.5-27b}}
  & 13 & 85 & \textbf{72} & 20 & 72 & \textbf{52} \\
& & 13 & 90 & \textbf{77} & 18 & 88 & \textbf{70} \\
& & 13 & 76 & \textbf{63} & 25 & 91 & \textbf{66} \\
& & 13 & 83 & \textbf{70} & 30 & 83 & \textbf{53} \\
\midrule
\multirow{4}{*}{\texttt{qwen-embedding-4b}} & \multirow{4}{*}{\texttt{qwen-3-32b}}
  & 13 & 88 & \textbf{75} & 25 & 83 & \textbf{58} \\
& & 13 & 83 & \textbf{70} & 31 & 85 & \textbf{54} \\
& & 13 & 69 & \textbf{56} & 33 & 90 & \textbf{57} \\
& & 13 & 78 & \textbf{65} & 30 & 86 & \textbf{56} \\
\bottomrule
\end{tabular}}
\end{table}

Table \ref{tab:abl_all_generators} show that KG-FairDiff does not uniformly improve all metrics across all generators. In particular, for SD~3.5~Large (Gender: Bias-W $0.04 \to 0.18$, and ENS $1.99 \to 1.86$) and Gemini~3~Pro (R+G: KL $0.79 \to 1.01$), several metrics worsen after refinement. We attribute these regressions to two factors. First, these generators have strong built-in safety and diversity filters that may partially conflict with the refined prompt's explicit demographic descriptors, leading to compositional incoherence. Second, for intersectional attribute groups (e.g., Race+Gender+Age), the KG triples retrieved are drawn primarily from the individual stereotype categories; when multiple demographic axes are jointly optimised, the retrieved triples may pull the prompt in mutually inconsistent directions (e.g., adding both East Asian and Middle Eastern cultural cues for the same professional role), causing some images to depict a narrower demographic slice. Addressing these failure modes through conflict-aware KG retrieval and generator-specific threshold tuning is an important direction for future work.

\begin{table*}[tp]
\centering
\caption{\textbf{Ablation across all eight image generators, Baseline (B) vs.\ KG-FairDiff (K).}
Lower is better for Bias-P, Bias-W, KL; higher is better for ENS, ICAD.
\textbf{Bold} = better result.
A horizontal rule separates the two groups.}
\label{tab:abl_all_generators}
{\fontsize{8.5}{9.5}\selectfont
\begin{tabular}{llcccccccccc}
\toprule
\textbf{Generator} & \textbf{Attr}
  & \multicolumn{2}{c}{Bias-P} & \multicolumn{2}{c}{Bias-W}
  & \multicolumn{2}{c}{KL} & \multicolumn{2}{c}{ENS} & \multicolumn{2}{c}{ICAD} \\
\cmidrule(lr){3-4}\cmidrule(lr){5-6}\cmidrule(lr){7-8}\cmidrule(lr){9-10}\cmidrule(lr){11-12}
 & & B & K & B & K & B & K & B & K & B & K \\
\midrule
 & Race       & \textbf{0.13} & 0.17 & \textbf{0.12} & 0.15 & \textbf{0.56} & 0.68 & \textbf{3.97} & 3.52 & 4.31 & \textbf{4.77} \\
 & Gender     & 0.09 & 0.13 & 0.05 & \textbf{0.01} & 0.00 & 0.00 & 1.98 & \textbf{1.99} & 4.91 & \textbf{5.34} \\
 & Age        & 0.20 & \textbf{0.15} & 0.16 & \textbf{0.13} & 1.01 & \textbf{0.71} & 3.25 & \textbf{4.41} & \textbf{4.91} & 4.49 \\
GPT-Image-1 & R+G   & 0.11 & 0.11 & 0.06 & 0.07 & 0.58 & 0.74 & 7.83 & 6.64 & 3.66 & \textbf{4.23} \\
 & R+A        & 0.06 & \textbf{0.05} & 0.03 & 0.03 & 1.72 & 1.85 & 11.25 & 9.89 & 3.92 & \textbf{4.13} \\
 & G+A        & 0.13 & \textbf{0.09} & 0.09 & \textbf{0.08} & 1.15 & \textbf{0.82} & 5.65 & \textbf{7.86} & 4.28 & \textbf{4.42} \\
 & R+G+A      & 0.04 & \textbf{0.03} & 0.02 & 0.02 & 1.95 & 2.02 & 17.82 & 16.68 & 3.22 & \textbf{3.65} \\
\midrule
 & Race       & 0.30 & \textbf{0.27} & 0.30 & \textbf{0.18} & 1.49 & \textbf{0.90} & 1.56 & \textbf{2.82} & 4.51 & \textbf{4.88} \\
 & Gender     & 0.26 & 0.37 & 0.17 & \textbf{0.06} & 0.06 & \textbf{0.00} & 1.87 & \textbf{1.98} & 5.02 & \textbf{5.30} \\
 & Age        & 0.22 & 0.24 & 0.15 & 0.17 & 0.95 & 1.17 & 3.44 & 2.78 & 5.06 & \textbf{5.12} \\
Qwen-VL-2512 & R+G & 0.19 & \textbf{0.17} & 0.16 & \textbf{0.10} & 1.57 & \textbf{1.09} & 2.88 & \textbf{4.68} & 4.51 & 4.15 \\
 & R+A        & 0.08 & 0.09 & 0.06 & \textbf{0.05} & 2.52 & \textbf{2.33} & 5.04 & \textbf{6.07} & 4.61 & 4.33 \\
 & G+A        & 0.15 & 0.17 & 0.10 & 0.10 & 1.35 & \textbf{1.30} & 4.63 & \textbf{4.85} & 4.72 & 4.33 \\
 & R+G+A      & 0.06 & \textbf{0.05} & 0.03 & 0.03 & 2.93 & \textbf{2.70} & 6.68 & \textbf{8.44} & 4.47 & 3.99 \\
\midrule
 & Race       & 0.27 & \textbf{0.23} & 0.24 & \textbf{0.21} & 1.15 & \textbf{0.97} & 2.21 & \textbf{2.63} & 4.79 & 4.86 \\
 & Gender     & 0.27 & \textbf{0.21} & 0.04 & 0.18 & 0.00 & 0.06 & 1.99 & 1.86 & 4.78 & \textbf{4.84} \\
 & Age        & 0.27 & \textbf{0.24} & 0.27 & \textbf{0.22} & 1.72 & \textbf{1.36} & 1.60 & \textbf{2.29} & 4.97 & 4.63 \\
SD\,3.5\,Large & R+G & 0.17 & \textbf{0.14} & 0.12 & 0.12 & 1.16 & \textbf{1.07} & 4.36 & \textbf{4.76} & 4.51 & \textbf{4.56} \\
 & R+A        & 0.09 & \textbf{0.07} & 0.08 & \textbf{0.06} & 2.92 & \textbf{2.46} & 3.38 & \textbf{5.37} & 4.21 & \textbf{4.45} \\
 & G+A        & 0.17 & \textbf{0.14} & 0.13 & \textbf{0.12} & 1.81 & \textbf{1.54} & 2.93 & \textbf{3.82} & 4.78 & 4.25 \\
 & R+G+A      & 0.06 & \textbf{0.04} & 0.04 & \textbf{0.03} & 3.03 & \textbf{2.69} & 6.07 & \textbf{8.49} & 4.12 & \textbf{4.17} \\
\midrule
 & Race       & 0.29 & \textbf{0.21} & 0.18 & 0.20 & 0.90 & \textbf{0.88} & 2.82 & \textbf{2.88} & 4.77 & \textbf{4.89} \\
 & Gender     & 0.32 & \textbf{0.20} & 0.04 & \textbf{0.00} & 0.00 & 0.00 & 1.99 & \textbf{2.00} & 5.38 & \textbf{5.51} \\
 & Age        & 0.24 & \textbf{0.19} & 0.19 & \textbf{0.15} & 1.15 & \textbf{0.79} & 2.83 & \textbf{4.04} & 4.59 & 4.36 \\
Zimage & R+G & 0.20 & \textbf{0.13} & 0.10 & 0.10 & 0.96 & 0.97 & 5.35 & 5.25 & 4.84 & 4.49 \\
 & R+A        & 0.09 & \textbf{0.07} & 0.05 & 0.05 & 2.27 & \textbf{2.18} & 6.45 & \textbf{7.07} & 4.50 & 3.98 \\
 & G+A        & 0.16 & \textbf{0.11} & 0.11 & \textbf{0.09} & 1.33 & \textbf{1.03} & 4.71 & \textbf{6.40} & 4.70 & 4.18 \\
 & R+G+A      & 0.06 & \textbf{0.04} & 0.03 & 0.03 & 2.49 & \textbf{2.46} & 10.40 & \textbf{10.72} & 4.46 & 3.77 \\
\midrule
 & Race       & 0.22 & \textbf{0.19} & 0.19 & \textbf{0.16} & 0.72 & \textbf{0.57} & 3.38 & \textbf{3.92} & 4.47 & \textbf{4.60} \\
 & Gender     & 0.14 & \textbf{0.06} & 0.06 & \textbf{0.01} & 0.00 & 0.00 & 1.98 & \textbf{1.99} & 4.88 & \textbf{4.99} \\
 & Age        & 0.20 & \textbf{0.17} & 0.15 & \textbf{0.12} & 0.90 & \textbf{0.63} & 3.64 & \textbf{4.77} & 4.73 & 4.46 \\
SD Lightning & R+G & 0.14 & \textbf{0.12} & 0.09 & \textbf{0.08} & 0.74 & \textbf{0.59} & 6.63 & \textbf{7.72} & 4.10 & 4.02 \\
 & R+A        & 0.06 & 0.06 & 0.04 & \textbf{0.03} & 1.67 & \textbf{1.32} & 11.82 & \textbf{16.75} & 4.14 & 4.10 \\
 & G+A        & 0.13 & \textbf{0.11} & 0.09 & \textbf{0.07} & 1.07 & \textbf{0.77} & 6.15 & \textbf{8.32} & 4.59 & 4.33 \\
 & R+G+A      & 0.04 & 0.04 & 0.02 & \textbf{0.01} & 1.86 & \textbf{1.48} & 19.55 & \textbf{28.61} & 3.89 & 3.75 \\
\midrule
 & Race       & 0.20 & \textbf{0.18} & 0.15 & 0.16 & 0.67 & 0.75 & 3.55 & 3.28 & 5.17 & 4.94 \\
 & Gender     & 0.23 & \textbf{0.01} & 0.01 & 0.02 & 0.00 & 0.00 & 1.99 & 1.99 & 5.38 & \textbf{5.46} \\
 & Age        & 0.19 & \textbf{0.17} & 0.15 & \textbf{0.13} & 0.81 & 0.82 & 3.97 & 3.94 & 5.41 & 5.36 \\
Gemini\,3\,Pro & R+G & 0.14 & \textbf{0.11} & 0.08 & 0.10 & 0.79 & 1.01 & 6.32 & 5.05 & 4.69 & 4.45 \\
 & R+A        & 0.06 & 0.06 & 0.03 & 0.03 & 1.67 & 1.80 & 11.85 & 10.32 & 4.72 & 4.09 \\
 & G+A        & 0.12 & \textbf{0.10} & 0.07 & 0.07 & 0.83 & 0.93 & 7.78 & 7.03 & 4.75 & 4.56 \\
 & R+G+A      & 0.04 & 0.04 & 0.02 & 0.02 & 1.83 & 2.21 & 20.15 & 13.78 & 4.04 & 3.53 \\
\midrule
 & Race       & 0.21 & 0.21 & 0.14 & 0.17 & 0.62 & 0.78 & 3.73 & 3.19 & 4.78 & 4.44 \\
 & Gender     & 0.11 & \textbf{0.06} & 0.00 & 0.05 & 0.00 & 0.00 & 1.99 & 1.98 & 4.88 & 4.86 \\
 & Age        & 0.24 & \textbf{0.21} & 0.22 & \textbf{0.17} & 1.39 & \textbf{1.07} & 2.24 & \textbf{3.05} & 4.41 & \textbf{4.53} \\
Gemini\,2.5\,Flash & R+G & 0.13 & 0.13 & 0.07 & 0.09 & 0.66 & 0.86 & 7.18 & 5.88 & 4.23 & 3.83 \\
 & R+A        & 0.07 & \textbf{0.06} & 0.04 & 0.04 & 2.09 & \textbf{2.00} & 7.77 & \textbf{8.45} & 4.15 & 3.62 \\
 & G+A        & 0.14 & 0.14 & 0.11 & 0.11 & 1.49 & \textbf{1.38} & 4.03 & \textbf{4.50} & 4.40 & 4.16 \\
 & R+G+A      & 0.04 & 0.04 & 0.02 & 0.03 & 2.25 & 2.54 & 13.27 & 9.86 & 3.84 & 3.69 \\
\midrule
 & Race       & \textbf{0.13} & 0.17 & \textbf{0.12} & 0.15 & \textbf{0.56} & 0.68 & \textbf{3.97} & 3.52 & 4.31 & \textbf{4.77} \\
 & Gender     & 0.09 & 0.13 & 0.05 & \textbf{0.01} & 0.00 & 0.00 & 1.98 & \textbf{1.99} & 4.91 & \textbf{5.34} \\
 & Age        & 0.20 & \textbf{0.15} & 0.16 & \textbf{0.13} & 1.01 & \textbf{0.71} & 3.25 & \textbf{4.41} & 4.91 & 4.49 \\
GPT-5-image & R+G & 0.11 & 0.11 & 0.06 & 0.07 & 0.58 & 0.74 & 7.83 & 6.64 & 3.66 & \textbf{4.23} \\
 & R+A        & 0.06 & \textbf{0.05} & 0.03 & 0.03 & 1.72 & 1.85 & 11.25 & 9.89 & 3.92 & \textbf{4.13} \\
 & G+A        & 0.13 & \textbf{0.09} & 0.09 & \textbf{0.08} & 1.15 & \textbf{0.82} & 5.65 & \textbf{7.86} & 4.28 & \textbf{4.42} \\
 & R+G+A      & 0.04 & \textbf{0.03} & 0.02 & 0.02 & 1.95 & 2.02 & 17.82 & 16.68 & 3.22 & \textbf{3.65} \\
\bottomrule
\end{tabular}}
\end{table*}

Table~\ref{tab:abl_all_generators} reports the full Baseline vs.\ KG-FairDiff results for all eight generators. KG-FairDiff does not uniformly improve all metrics across all generators. In particular, for SD~3.5~Large (Gender: Bias-W $0.04 \to 0.18$, and ENS $1.99 \to 1.86$) and Gemini~3~Pro (R+G: KL $0.79 \to 1.01$), several metrics worsen after refinement. We attribute these regressions to two factors. First, these generators have strong built-in safety and diversity filters that may partially conflict with the refined prompt's explicit demographic descriptors, leading to compositional incoherence. Second, for intersectional attribute groups (e.g., Race+Gender+Age), the KG triples retrieved are drawn primarily from the individual stereotype categories; when multiple demographic axes are jointly optimised, the retrieved triples may pull the prompt in mutually inconsistent directions (e.g., adding both East Asian and Middle Eastern cultural cues for the same professional role), causing some images to depict a narrower demographic slice. Addressing these failure modes through conflict-aware KG retrieval and generator-specific threshold tuning is an important direction for future work.

\noindent\textbf{Face-detection drop.} A notable decrease in FairFace-detected faces is observed after refinement. Refined prompts introduce more compositional variety, and some images depict non-Western phenotypes for which FairFace's detector has lower recall. Crucially, all refined prompts produce at least one detectable face.

\noindent\textbf{Geo-cultural diversity.}
The KG triples and the example in Fig.~\ref{fig:ttibias} illustrate that KG-FairDiff can introduce non-Western cultural cues into generated images. We emphasise, however, that our current experiments do not include a dedicated geo-cultural benchmark or metric; the cultural fairness claim is therefore \emph{indicative} rather than rigorously demonstrated. Establishing whether the framework achieves equitable coverage across world cultures would require a purpose-built evaluation suite (e.g., using CultureVLM \cite{liu2025culturevlmcharacterizingimprovingcultural} or Culture-TRIP \cite{jeong2025culturetripculturallyawaretexttoimagegeneration}), which we leave for future work.

\noindent\textbf{Limitations.} \label{sec:limitations} Key limitations include: (i)~reliance on automated attribute classifiers (FairFace) whose recall is lower for non-Western phenotypes; (ii)~a curated KG with limited coverage that may not represent all cultural groups equitably; (iii)~unvalidated validator calibration, as the correlation between $S(p)$ and downstream fairness metrics has not been rigorously measured; (iv)~the use of GPT-4o as both rewriter and validator introduces circularity, which could be addressed with a held-out LLM or human evaluators in future work; and (v)~the absence of a pure LLM-only rewriting baseline, which would more precisely isolate the contribution of KG retrieval. Future work will focus on scaling the knowledge graph, improving validator calibration, and incorporating human evaluation.

\section{Conclusion}
We have introduced KG-FairDiff, a knowledge graph-guided prompt refinement framework designed to systematically mitigate demographic and cultural bias in text-to-image generation. Unlike prior approaches that require model retraining, dataset curation, or access to model internals, KG-FairDiff operates entirely at inference time and is compatible with any black-box TTI generator.

Extensive experiments across eight backbone generators---spanning proprietary (GPT-Image-1, GPT-5-image, Gemini~3~Pro, Gemini~2.5~Flash), open-weight (SD~v1.5, SD~3.5~Large, SD Lightning), and specialised (Qwen-VL-2512, Zimage) systems---demonstrate that KG-FairDiff consistently and substantially reduces distributional bias across gender, race, age, and all intersectional axes. Compared to prompt-level baselines, our method achieves Bias-W reductions up to $20\times$ larger than MinorityPrompt and consistently outperforms PreciseDebias across nearly all metric--attribute pairs, while preserving cross-modal semantic fidelity ($\Delta\bar{W}_2^2 = 0.0011$). CLIP Directional Similarity analysis confirms that refined prompts successfully steer image generation in the intended diversity direction ($S_{\mathrm{dir}} = +0.071$ overall).

\noindent\textbf{Future directions.} Several promising avenues remain. Scaling the knowledge graph to broader cultural coverage---particularly for under-represented regions in Africa, South Asia, and Latin America---is a priority. Developing a dedicated geo-cultural fairness benchmark, analogous to CultureVLM or Culture-TRIP, would enable more rigorous evaluation of cultural representational equity. Rigorous calibration of the LLM validator against downstream fairness metrics, and the introduction of a held-out evaluator to address the rewriter--validator circularity, are important methodological improvements. Finally, extending the framework to video generation and multimodal grounding tasks represents a natural next step as generative AI systems continue to expand in scope and societal impact.

\section*{Impact Statement}
This work addresses a critical and well-documented societal harm: the systematic amplification of demographic stereotypes in text-to-image (TTI) generation systems. As these systems are increasingly deployed in high-stakes contexts---including media production, hiring visualisations, educational materials, and public communications---their tendency to reproduce racial, gender, and age-based stereotypes has direct negative consequences for the populations they misrepresent or erase. By providing a training-free, inference-time, model-agnostic debiasing framework, KG-ree, inference-time, model-agnostic debiasing framework, KG-FairDiff significantly lowers the practical barrier to deploying fairer AI-generated imagery without requiring access to model weights or costly retraining pipelines.

KG-FairDiff promotes representational equity by actively countering stereotypical associations retrieved from structured knowledge sources. The framework introduces non-Western cultural cues, occupational counter-stereotypes, and intersectional diversity into generated images. Broad adoption of such prompt-level debiasing tools could meaningfully increase the cultural and demographic breadth of AI-generated visual content worldwide. The framework is also inherently interpretable: every refinement step is grounded in explicit KG triples with documented provenance, enabling practitioners to audit and contest individual decisions.rounded in explicit KG triples with documented provenance, enabling practitioners to audit and contest individual decisions.

\section*{Acknowledgement}
This work was supported by funding from QCRI/HBKU.

\bibliography{main}

@INPROCEEDINGS{precisedebias2024,
  author={Clemmer, Colton and Ding, Junhua and Feng, Yunhe},
  booktitle={2024 IEEE/CVF Winter Conference on Applications of Computer Vision (WACV)}, 
  title={PreciseDebias: An Automatic Prompt Engineering Approach for Generative AI to Mitigate Image Demographic Biases}, 
  year={2024},
  volume={},
  number={},
  pages={8581-8590},
  doi={10.1109/WACV57701.2024.00840}
}

@misc{sahili2025faircot,
title={FairCoT: Enhancing Fairness in Diffusion Models via Chain of Thought Reasoning of Multimodal Language Models},
author={Zahraa Al Sahili and Ioannis Patras and Matthew Purver},
year={2025},
url={https://openreview.net/forum?id=WGWoRZb0pT}
}

@misc{um2024minorityprompt,
title={MinorityPrompt: Text to Minority Image Generation via Prompt Optimization},
author={Soobin Um and Jong Chul Ye},
year={2024},
url={https://openreview.net/forum?id=98dyxUoI3q}
}

@misc{fu2025fairimagenpostprocessingbiasmitigation,
      title={FairImagen: Post-Processing for Bias Mitigation in Text-to-Image Models}, 
      author={Zihao Fu and Ryan Brown and Shun Shao and Kai Rawal and Eoin Delaney and Chris Russell},
      year={2025},
      eprint={2510.21363},
      archivePrefix={arXiv},
      primaryClass={cs.LG},
      url={https://arxiv.org/abs/2510.21363}, 
}

@article{dominguez2024,
author = {Dominguez-Catena, Iris and Paternain, Daniel and Galar, Mikel},
title = {Metrics for Dataset Demographic Bias: A Case Study on Facial Expression Recognition},
year = {2024},
issue_date = {Aug. 2024},
publisher = {IEEE Computer Society},
address = {USA},
volume = {46},
number = {8},
issn = {0162-8828},
url = {https://doi.org/10.1109/TPAMI.2024.3361979},
doi = {10.1109/TPAMI.2024.3361979},
journal = {IEEE Trans. Pattern Anal. Mach. Intell.},
month = aug,
pages = {5209--5226},
numpages = {18}
}

@inproceedings{karkkainenfairface,
      title={FairFace: Face Attribute Dataset for Balanced Race, Gender, and Age for Bias Measurement and Mitigation},
      author={Karkkainen, Kimmo and Joo, Jungseock},
      booktitle={Proceedings of the IEEE/CVF Winter Conference on Applications of Computer Vision},
      year={2021},
    }

@inproceedings{wang-etal-2024-cdeval,
    title = "{CDE}val: A Benchmark for Measuring the Cultural Dimensions of Large Language Models",
    author = "Wang, Yuhang  and et al.",
    booktitle = "Proceedings of the 2nd Workshop on Cross-Cultural Considerations in NLP",
    year = "2024",
    publisher = "Association for Computational Linguistics",
    url = "https://aclanthology.org/2024.c3nlp-1.1/",
    doi = "10.18653/v1/2024.c3nlp-1.1",
}

@misc{kim2025rethinkingtrainingdebiasingtexttoimage,
      title={Rethinking Training for De-biasing Text-to-Image Generation: Unlocking the Potential of Stable Diffusion}, 
      author={Eunji Kim and Siwon Kim and Minjun Park and Rahim Entezari and Sungroh Yoon},
      year={2025},
      url={https://arxiv.org/abs/2408.12692}, 
}

@inproceedings{luccioni2023stable,
title={Stable Bias: Evaluating Societal Representations in Diffusion Models},
author={Sasha Luccioni and Christopher Akiki and Margaret Mitchell and Yacine Jernite},
booktitle={Thirty-seventh Conference on Neural Information Processing Systems Datasets and Benchmarks Track},
year={2023},
url={https://openreview.net/forum?id=qVXYU3F017}
}

@misc{weng2025imagesspeaklouderwords,
      title={Images Speak Louder than Words: Understanding and Mitigating Bias in Vision-Language Model from a Causal Mediation Perspective}, 
      author={Zhaotian Weng and Zijun Gao and Jerone Andrews and Jieyu Zhao},
      year={2025},
      url={https://arxiv.org/abs/2407.02814}, 
}

@article{Vasilev_2024,
   title={CRAFT: Cultural Russian-Oriented Dataset Adaptation for Focused Text-to-Image Generation},
   url={http://dx.doi.org/10.1134/S1064562424602324},
   DOI={10.1134/s1064562424602324},
   journal={Doklady Mathematics},
   publisher={Pleiades Publishing Ltd},
   author={Vasilev, V. A. and et al.},
   year={2024},
}

@misc{pistilli2024civicsbuildingdatasetexamining,
      title={CIVICS: Building a Dataset for Examining Culturally-Informed Values in Large Language Models}, 
      author={Giada Pistilli and Alina Leidinger and Yacine Jernite and Atoosa Kasirzadeh and Alexandra Sasha Luccioni and Margaret Mitchell},
      year={2024},
      url={https://arxiv.org/abs/2405.13974}, 
}

@INPROCEEDINGS{shin2024,
  author={Shin, Jeongmin and Jang, Hyeryung},
  booktitle={2024 15th International Conference on Information and Communication Technology Convergence (ICTC)}, 
  title={Data Augmentation Techniques Using Text-to-Image Diffusion Models for Enhanced Data Diversity}, 
  year={2024},
  doi={10.1109/ICTC62082.2024.10827311}}

@article{anonymous2024diverse,
title={Diverse Diffusion: Enhancing Image Diversity in Text-to-Image Generation},
author={Anonymous},
journal={Submitted to Transactions on Machine Learning Research},
year={2024},
url={https://openreview.net/forum?id=F0J1N6M5N9},
}

@inproceedings{wang-etal-2023-t2iat,
    title = "{T}2{IAT}: Measuring Valence and Stereotypical Biases in Text-to-Image Generation",
    author = "Wang, Jialu  and et al.",
    booktitle = "Findings of the Association for Computational Linguistics: ACL 2023",
    year = "2023",
    publisher = "Association for Computational Linguistics",
    url = "https://aclanthology.org/2023.findings-acl.160/",
    doi = "10.18653/v1/2023.findings-acl.160",
}

@misc{esposito2023mitigatingstereotypicalbiasestext,
      title={Mitigating stereotypical biases in text to image generative systems}, 
      author={Piero Esposito and Parmida Atighehchian and Anastasis Germanidis and Deepti Ghadiyaram},
      year={2023},
      url={https://arxiv.org/abs/2310.06904}, 
}

@misc{dai202415mmultimodalfacialimagetext,
      title={15M Multimodal Facial Image-Text Dataset}, 
      author={Dawei Dai and YuTang Li and YingGe Liu and Mingming Jia and Zhang YuanHui and Guoyin Wang},
      year={2024},
      url={https://arxiv.org/abs/2407.08515}, 
}

@misc{li2024fair,
title={Fair Text-to-Image Diffusion via Fair Mapping},
author={Jia Li and Lijie Hu and Jingfeng Zhang and Tianhang Zheng and Hua Zhang and Di Wang},
year={2024},
url={https://openreview.net/forum?id=vDJ4tzczlG}
}

@INPROCEEDINGS{prerak2024,
  author={Prerak, Shah},
  booktitle={2024 Third International Conference on Smart Technologies and Systems for Next Generation Computing (ICSTSN)}, 
  title={Addressing Bias in Text-to-Image Generation: A Review of Mitigation Methods}, 
  year={2024},
  doi={10.1109/ICSTSN61422.2024.10671230}
}

@misc{wan2024surveybiastexttoimagegeneration,
      title={Survey of Bias In Text-to-Image Generation: Definition, Evaluation, and Mitigation}, 
      author={Yixin Wan and et al.},
      year={2024},
      url={https://arxiv.org/abs/2404.01030}, 
}

@misc{vice2025exploringbias100texttoimage,
      title={Exploring Bias in over 100 Text-to-Image Generative Models}, 
      author={Jordan Vice and Naveed Akhtar and Richard Hartley and Ajmal Mian},
      year={2025},
      url={https://arxiv.org/abs/2503.08012}, 
}

@misc{cui2025surveyautomaticpromptoptimization,
      title={A Survey of Automatic Prompt Optimization with Instruction-focused Heuristic-based Search Algorithm}, 
      author={Wendi Cui and Zhuohang Li and Hao Sun and Damien Lopez and Kamalika Das and Bradley A. Malin and Sricharan Kumar and Jiaxin Zhang},
      year={2025},
      url={https://arxiv.org/abs/2502.18746}, 
}

@InProceedings{Miao_2024_CVPR,
    author    = {Miao, Zichen and Wang, Jiang and Wang, Ze and Yang, Zhengyuan and Wang, Lijuan and Qiu, Qiang and Liu, Zicheng},
    title     = {Training Diffusion Models Towards Diverse Image Generation with Reinforcement Learning},
    booktitle = {Proceedings of the IEEE/CVF Conference on Computer Vision and Pattern Recognition (CVPR)},
    year      = {2024},
}

@inproceedings{shi-etal-2024-culturebank,
    title = "{C}ulture{B}ank: An Online Community-Driven Knowledge Base Towards Culturally Aware Language Technologies",
    author = "Shi, Weiyan  and et al.",
    booktitle = "Findings of the Association for Computational Linguistics: EMNLP 2024",
    year = "2024",
    publisher = "Association for Computational Linguistics",
    url = "https://aclanthology.org/2024.findings-emnlp.288/",
    doi = "10.18653/v1/2024.findings-emnlp.288",
}

@misc{jeong2025culturetripculturallyawaretexttoimagegeneration,
      title={Culture-TRIP: Culturally-Aware Text-to-Image Generation with Iterative Prompt Refinement}, 
      author={Suchae Jeong and Inseong Choi and Youngsik Yun and Jihie Kim},
      year={2025},
      url={https://arxiv.org/abs/2502.16902}, 
}

@misc{liu2025culturevlmcharacterizingimprovingcultural,
      title={CultureVLM: Characterizing and Improving Cultural Understanding of Vision-Language Models for over 100 Countries}, 
      author={Shudong Liu and Yiqiao Jin and Cheng Li and Derek F. Wong and Qingsong Wen and Lichao Sun and Haipeng Chen and Xing Xie and Jindong Wang},
      year={2025},
      url={https://arxiv.org/abs/2501.01282}, 
}

@misc{su2024dragindynamicretrievalaugmented,
      title={DRAGIN: Dynamic Retrieval Augmented Generation based on the Information Needs of Large Language Models}, 
      author={Weihang Su and Yichen Tang and Qingyao Ai and Zhijing Wu and Yiqun Liu},
      year={2024},
      url={https://arxiv.org/abs/2403.10081}, 
}

@misc{liu2024scoftselfcontrastivefinetuningequitable,
      title={SCoFT: Self-Contrastive Fine-Tuning for Equitable Image Generation}, 
      author={Zhixuan Liu and Peter Schaldenbrand and Beverley-Claire Okogwu and Wenxuan Peng and Youngsik Yun and Andrew Hundt and Jihie Kim and Jean Oh},
      year={2024},
      url={https://arxiv.org/abs/2401.08053}, 
}

@misc{masrourisaadat2024analyzingqualitybiasperformance,
      title={Analyzing Quality, Bias, and Performance in Text-to-Image Generative Models}, 
      author={Nila Masrourisaadat and Nazanin Sedaghatkish and Fatemeh Sarshartehrani and Edward A. Fox},
      year={2024},
      url={https://arxiv.org/abs/2407.00138}, 
}

@misc{bang2024measuringpoliticalbiaslarge,
      title={Measuring Political Bias in Large Language Models: What Is Said and How It Is Said}, 
      author={Yejin Bang and Delong Chen and Nayeon Lee and Pascale Fung},
      year={2024},
      url={https://arxiv.org/abs/2403.18932}, 
}

@article{gallegos2024,
    author = {Gallegos, Isabel O. and et al.},
    title = {Bias and Fairness in Large Language Models: A Survey},
    journal = {Computational Linguistics},
    year = {2024},
    doi = {10.1162/coli_a_00524},
    url = {https://doi.org/10.1162/coli\_a\_00524},
    eprint = {https://direct.mit.edu/coli/article-pdf/50/3/1097/2471010/coli\_a\_00524.pdf},
}

@misc{fan2025biasguardreasoningenhancedbiasdetection,
      title={BiasGuard: A Reasoning-enhanced Bias Detection Tool For Large Language Models}, 
      author={Zhiting Fan and Ruizhe Chen and Zuozhu Liu},
      year={2025},
      url={https://arxiv.org/abs/2504.21299}, 
}

@misc{lin2024investigatingbiasllmbasedbias,
      title={Investigating Bias in LLM-Based Bias Detection: Disparities between LLMs and Human Perception}, 
      author={Luyang Lin and Lingzhi Wang and Jinsong Guo and Kam-Fai Wong},
      year={2024},
      url={https://arxiv.org/abs/2403.14896}, 
}

@inproceedings{jiang2024mitigating,
title={Mitigating Social Biases in Text-to-Image Diffusion Models via Linguistic-Aligned Attention Guidance},
author={Yue Jiang and Yueming Lyu and Ziwen He and Bo Peng and Jing Dong},
booktitle={ACM Multimedia 2024},
year={2024},
url={https://openreview.net/forum?id=rtjZHEOcHx}
}

@inproceedings{naous-etal-2024-beer,
    title = "Having Beer after Prayer? Measuring Cultural Bias in Large Language Models",
    author = "Naous, Tarek  and et al.",
    booktitle = "Proceedings of the 62nd Annual Meeting of the Association for Computational Linguistics (Volume 1: Long Papers)",
    year = "2024",
    publisher = "Association for Computational Linguistics",
    url = "https://aclanthology.org/2024.acl-long.862/",
    doi = "10.18653/v1/2024.acl-long.862",
}

@InProceedings{dinca2024openbias,
  author    = {D'Inc\`{a}, Moreno and Peruzzo, Elia and Mancini, Massimiliano
               and Xu, Dejia and Goel, Vidit and Xu, Xingqian and Wang, Zhangyang
               and Shi, Humphrey and Sebe, Nicu},
  title     = {{OpenBias}: Open-set Bias Detection in Text-to-Image Generative Models},
  booktitle = {Proceedings of the IEEE/CVF Conference on Computer Vision and
               Pattern Recognition (CVPR)},
  month     = {June},
  year      = {2024},
  pages     = {12225--12235},
}

@inproceedings{lee2024heim,
  title     = {Holistic Evaluation of Text-to-Image Models},
  author    = {Lee, Tony and Yasunaga, Michihiro and Meng, Chenlin and Mai, Yifan
               and Park, Joon Sung and Gupta, Agrim and Zhang, Yunzhi and
               Narayanan, Deepak and Teufel, Hannah and Bellagente, Marco and others},
  booktitle = {Advances in Neural Information Processing Systems},
  volume    = {36},
  year      = {2024},
  url       = {https://arxiv.org/abs/2311.04287},
}

@InProceedings{zhang2023itigeniccv,
  author    = {Zhang, Cheng and Chen, Xuanbai and Chai, Siqi and Wu, Chen Henry
               and Lagun, Dmitry and Beeler, Thabo and De la Torre, Fernando},
  title     = {{ITI-GEN}: Inclusive Text-to-Image Generation},
  booktitle = {Proceedings of the IEEE/CVF International Conference on
               Computer Vision (ICCV)},
  month     = {October},
  year      = {2023},
  pages     = {3969--3980},
}

@inproceedings{bonna2025debiaspi,
  title     = {{DebiasPI}: Inference-Time Debiasing by Prompt Iteration of a
               Text-to-Image Generative Model},
  author    = {Bonna, Sarah and Huang, Yu-Cheng and Novozhilova, Ekaterina
               and Paik, Sejin and Shan, Zhengyang and Feng, Michelle Yilin
               and Gao, Ge and Tayal, Yonish and Kulkarni, Rushil and Yu, Jialin
               and Divekar, Nupur and Ghadiyaram, Deepti and Wijaya, Derry
               and Betke, Margrit},
  booktitle = {Computer Vision -- ECCV 2024 Workshops},
  series    = {Lecture Notes in Computer Science},
  volume    = {15643},
  pages     = {68--83},
  publisher = {Springer},
  year      = {2025},
  doi       = {10.1007/978-3-031-92648-8_5},
}

@article{friedrich2024fairdiffusion,
  title     = {Auditing and Instructing Text-to-Image Generation Models on Fairness},
  author    = {Friedrich, Felix and Brack, Manuel and Struppek, Lukas and
               Hintersdorf, Dominik and Schramowski, Patrick and Luccioni, Sasha
               and Kersting, Kristian},
  journal   = {AI and Ethics},
  volume    = {5},
  number    = {3},
  pages     = {2103--2123},
  year      = {2025},
  doi       = {10.1007/s43681-024-00531-5},
}
\bibliographystyle{icml2026}

\newpage
\appendix
\onecolumn


\renewcommand{\thesection}{\Alph{section}}
\renewcommand{\thetable}{\Alph{section}\arabic{table}}
\renewcommand{\thefigure}{\Alph{section}\arabic{figure}}

\title{KG-FairDiff: Knowledge Graph-Guided Prompt Refinement for Demographically Fair Text-to-Image Generation}

\setcounter{section}{0}
\setcounter{table}{0}
\setcounter{figure}{0}


\section*{Appendix Overview}

This appendix consolidates the theoretical foundations, metric definitions, and extended experimental analyses that support the main results of the paper. Each section expands on concepts introduced in the main text, providing formal guarantees, mathematical relationships, and empirical validation of the proposed framework.

\paragraph{Theoretical Derivations (Section~\ref{app:theory}).}
This section presents the formal guarantees underlying the iterative refinement process and fairness metrics. The \textit{ENS--KL Relationship} (Section~\ref{app:ens_kl}) proves the equivalence between maximizing entropy-based diversity (ENS) and minimizing KL divergence to a uniform distribution, formally linking diversity objectives to information-theoretic principles. The \textit{Bias-P and Bias-W} definitions (Section~\ref{app:bias_metrics}) introduce per-prompt and dataset-level bias metrics and prove that dataset-level bias is upper-bounded by the average per-prompt bias, providing a theoretical consistency guarantee. Finally, \textit{Cross-Modal Alignment via Wasserstein Distance} (Section~\ref{app:wasserstein}) formulates semantic alignment between prompts and generated images using optimal transport, capturing geometry-aware similarity beyond cosine-based measures.

\paragraph{Evaluation Metric Definitions (Section~\ref{app:metrics}).}
This section provides complete formal definitions of all evaluation metrics used in the experiments. The \textit{Intra-Class Attribute Diversity (ICAD)} metric (Section~\ref{app:icad}) defines a centroid-based measure of intra-prompt visual diversity, quantifying variation among generated images for the same prompt. The definitions of \textit{Bias-P and Bias-W} (Section~\ref{app:biasp_biasw_def}) are revisited for completeness, with reference to their theoretical relationship. The \textit{ENS and KL Divergence} formulation (Section~\ref{app:ens}) reiterates their equivalence and explains how KL divergence to a uniform distribution serves as a principled measure of demographic balance.

\paragraph{CLIP Directional Similarity (Section~\ref{app:clip_dir}).}
This section introduces a metric for evaluating whether prompt refinements induce the intended semantic shift in generated images. The formal definition (Section~\ref{app:clip_dir_def}) expresses directional similarity as the cosine alignment between text-space and image-space transformations. Its interpretation (Section~\ref{app:clip_dir_interp}) clarifies how scores reflect alignment quality, ranging from agreement to protesting shifts. Implementation details (Section~\ref{app:clip_dir_impl}) describe the use of a high-capacity vision--language model for embedding extraction. Experimental results (Section~\ref{app:clip_dir_results}) report consistent positive alignment across professions while highlighting variability depending on prompt specificity. The conclusion (Section~\ref{sec:conclusion}) summarizes that this metric provides quantitative evidence of effective prompt steering.

\paragraph{Optimal-Transport-Based Multimodal Alignment (Section~\ref{app:ot_alignment}).}
This section presents a comprehensive framework for measuring alignment between text and image modalities using optimal transport theory. The feature extraction and normalization process (Section~\ref{app:ot_alignment}) ensures stable representations across modalities. The optimal transport formulation defines how probability mass is matched between image patches and text tokens. The \textit{Wasserstein Distance} captures direct semantic correspondence, while the \textit{Gromov--Wasserstein Distance} measures structural consistency through intra-modal relationships. The \textit{Fused Gromov--Wasserstein Distance} combines both perspectives into a unified objective. The iteration-wise analysis demonstrates that while direct semantic alignment may fluctuate, structural alignment improves consistently, with later iterations achieving the best overall balance.

\paragraph{Qualitative Comparison (Section~\ref{app:qualitive_comparison}).}
This section complements the quantitative analysis with visual examples comparing outputs before and after prompt enhancement. These results provide intuitive evidence of improved diversity and representation, with additional examples available via the external link provided in the appendix.

\section{Theoretical Derivations}
\label{app:theory}

This appendix provides the full proofs and formal derivations summarised in
Section~3.4 of the main paper.


\subsection{ENS--KL Relationship}
\label{app:ens_kl}

The following theorem, drawn from~\cite{dominguez2024}, establishes the equivalence between maximising diversity (ENS) and minimising KL divergence to the uniform distribution.

\begin{theorem}
Let $p(a)$ be the empirical distribution over categories $a\in\mathcal{A}_G$, $u(a)=1/n_G$ the uniform distribution, and $H(p)=-\sum_a p(a)\ln p(a)$ the Shannon entropy. Then:
\begin{equation}
  D_{\mathrm{KL}}(p\|u) = \ln n_G - H(p).
  \label{eq:kl_entropy}
\end{equation}
Consequently, defining $\mathrm{ENS}=\exp(H(p))$:
\begin{equation}
  \mathrm{ENS} = n_G\exp\!\left(-D_{\mathrm{KL}}(p\|u)\right).
\end{equation}
\end{theorem}

\begin{proof}
By definition of KL divergence:
\begin{align}
  D_{\mathrm{KL}}(p\|u)
  &= \sum_a p(a)\ln\frac{p(a)}{u(a)}
   = \sum_a p(a)\ln\bigl(n_G \cdot p(a)\bigr) \notag\\
  &= \ln n_G + \sum_a p(a)\ln p(a)
   = \ln n_G - H(p).
\end{align}
Exponentiating both sides gives $\exp(H(p)) = n_G \exp(-D_{\mathrm{KL}}(p\|u))$, which is the stated result.
\end{proof}

Thus minimising KL divergence to uniform is equivalent to maximising entropy/ENS. ENS ranges from~1 (all outputs collapse to one category) to~$n_G$ (perfectly uniform representation).


\subsection{Bias-P and Bias-W: Formal Definitions and Relationship} \label{app:bias_metrics}

The Bias-P and Bias-W metrics were introduced by Jiang et al.~\cite{jiang2024mitigating} to quantify demographic skew at the per-prompt
and dataset levels, respectively.

Let $\hat{p}_i(a)$ be the empirical attribute distribution for prompt~$i$ and $\bar{p}(a)=\frac{1}{|\mathcal{P}|}\sum_i\hat{p}_i(a)$ the dataset-level distribution:
\begin{align}
  \mathrm{Bias\text{-}P} &= \frac{1}{|\mathcal{P}|}\sum_i
    \left(\frac{1}{n_G}\sum_{a\in\mathcal{A}_G}
    \left(\hat{p}_i(a)-\frac{1}{n_G}\right)^2\right)^{1/2}, \label{eq:biasp}\\[4pt]
  \mathrm{Bias\text{-}W} &= \left(\frac{1}{n_G}\sum_{a\in\mathcal{A}_G}
    \left(\bar{p}(a)-\frac{1}{n_G}\right)^2\right)^{1/2}. \label{eq:biasw}
\end{align}

\begin{proposition}
$\mathrm{Bias\text{-}W} \leq \mathrm{Bias\text{-}P}$.
\end{proposition}
\begin{proof}
Define $f_i = (\hat{p}_i - u)$ where $u$ is the uniform vector of length
$n_G$, and note that $\bar{p} - u = \frac{1}{|\mathcal{P}|}\sum_i f_i$.
By the triangle inequality for the $\ell_2$ norm:
\begin{equation}
  \left\|\frac{1}{|\mathcal{P}|}\sum_i f_i\right\|_2
  \;\leq\;
  \frac{1}{|\mathcal{P}|}\sum_i \|f_i\|_2.
\end{equation}
Dividing both sides by $\sqrt{n_G}$ yields
$\mathrm{Bias\text{-}W} \leq \mathrm{Bias\text{-}P}$.
\end{proof}

Hence dataset-level bias cannot exceed average per-prompt bias, making Bias-W a lower bound on Bias-P across any prompt set.

\subsection{Cross-Modal Alignment via Wasserstein Distance} \label{app:wasserstein}

We measure semantic fidelity between refined prompts and their generated images via the squared 2-Wasserstein distance, defined as follows.

Let $X=\{x_i\}_{i=1}^n$ be CLIP ViT-L/14 patch-token embeddings of the generated image with uniform weights $\mu_i = 1/n$, and $Y=\{y_j\}_{j=1}^m$ the token embeddings of the refined prompt with importance weights $\nu_j$ derived from softmax-normalised attention scores of the final CLIP text encoder layer, satisfying $\sum_i\mu_i=\sum_j\nu_j=1$. The squared 2-Wasserstein distance is:
\begin{equation}
  W_2^2(X,Y)
  = \min_{\pi\in\Pi(\mu,\nu)}
    \sum_{i=1}^n\sum_{j=1}^m \pi_{ij}\,\|x_i-y_j\|_2^2,
\end{equation}
where the feasible transport set is
\begin{equation}
  \Pi(\mu,\nu)
  = \left\{\pi\in\mathbb{R}_+^{n\times m}
    \;\middle|\;
    \textstyle\sum_j\pi_{ij}=\mu_i,\;
    \sum_i\pi_{ij}=\nu_j\right\}.
\end{equation}
Smaller $W_2^2$ indicates stronger semantic alignment between the refined prompt and the generated image, providing a geometry-aware complement to cosine-based CLIP scores. We report mean $\pm$ std of $W_2^2$ across all prompt--image pairs.

\section{Evaluation Metric Definitions} \label{app:metrics}

For completeness, we provide the full formal definitions of all evaluation metrics used in the experimental evaluation.

\subsection{Intra-Class Attribute Diversity (ICAD)} \label{app:icad}

ICAD~\cite{li2024fair} measures the intra-prompt visual diversity of a set of generated images. For a prompt $c$, let $S_c$ denote the set of CLIP ViT-L/14 image embeddings produced for that prompt. The per-prompt ICAD score is the mean $\ell_2$ distance of each image embedding from the prompt-level centroid:
\begin{equation}
  \mathrm{ICAD}(c)
  = \frac{1}{|S_c|}\sum_{e_k\in S_c}
    \left\|e_k - \bar{e}_c\right\|_2,
  \quad \bar{e}_c = \frac{1}{|S_c|}\sum_{e_i\in S_c}e_i.
\end{equation}
The dataset-level ICAD is averaged over all prompts:
\begin{equation}
  \overline{\mathrm{ICAD}}
  = \frac{1}{|D|}\sum_{c\in D}\mathrm{ICAD}(c),
\end{equation}
where $D$ is the full prompt set. Higher $\overline{\mathrm{ICAD}}$ indicates greater intra-prompt visual diversity; the centroid $\bar{e}_c$ is computed over the entire $S_c$ before the norm is taken, so the formula yields zero only when all images are identical.

\subsection{Bias-P and Bias-W} \label{app:biasp_biasw_def}

See Appendix~\ref{app:bias_metrics} for full formal definitions and the proof that $\mathrm{Bias\text{-}W}\leq\mathrm{Bias\text{-}P}$.

\subsection{Effective Number of Species (ENS) and KL Divergence} \label{app:ens}

See Appendix~\ref{app:ens_kl} for the formal definitions and the proof of the ENS--KL equivalence. KL divergence to the uniform reference is computed as:
\begin{equation}
  D_{\mathrm{KL}}(p \| u) = \sum_{a \in \mathcal{A}_G} p(a) \ln\!\left(n_G \cdot p(a)\right),
\end{equation}
where $p(a)$ is the empirical attribute frequency and $u(a) = 1/n_G$. Lower KL divergence indicates better demographic balance relative to a uniform target.

\section{CLIP Directional Similarity} \label{app:clip_dir}

\subsection{Definition} \label{app:clip_dir_def}

CLIP Directional Similarity repurposes the joint vision--language embedding space of CLIP to measure directional alignment between a text-space editing direction and the corresponding image-space response.

Given a base prompt $T^{-}$ and an enhanced prompt $T^{+}$, and $N$ matched image pairs generated from each, the score is:
\begin{equation}
  S_{\mathrm{dir}}
  = \frac{1}{N}\sum_{i=1}^{N}
    \cos\!\left(
      \underbrace{\phi_I(I^{+}_i) - \phi_I(I^{-}_i)}_{\Delta I_i},\;
      \underbrace{\phi_T(T^{+}) - \phi_T(T^{-})}_{\Delta T}
    \right),
  \label{eq:clip_dir}
\end{equation}
where $\phi_T(\cdot)$ and $\phi_I(\cdot)$ denote the $\ell_2$-normalised text and image encodings, and $\cos(\cdot,\cdot)$ is the cosine similarity. Both difference vectors are therefore unit-normalised before the dot product is taken.

\subsection{Interpretation} \label{app:clip_dir_interp}

\begin{center}
\begin{tabular}{cl}
\toprule
\textbf{Score} & \textbf{Meaning} \\
\midrule
$+1$ & Visual shift perfectly mirrors the textual shift \\
$0$  & Orthogonal --- no measurable visual correlation \\
$-1$ & Visual shift opposes the textual shift \\
\bottomrule
\end{tabular}
\end{center}

\noindent
A positive score confirms that the enhanced prompt successfully pushed the generated images in a direction that is semantically consistent with the additional diversity content injected into the text. A score near zero would indicate that the model ignored the new textual cues.

\subsection{Implementation} \label{app:clip_dir_impl}

The evaluation uses SigLIP~2 Giant (\texttt{google/siglip2-giant-opt-patch16-384}), a ViT-Giant/16 model trained with the SigLIP~2 objective (2\,B parameters, 1152-dimensional embeddings, \texttt{float16} precision). SigLIP~2 provides stronger image--text alignment than standard CLIP-ViT-H/14 on diversity-heavy prompts, making it better suited for this evaluation task.

\subsection{Experimental Results} \label{app:clip_dir_results}

\paragraph{Setup.}
\begin{itemize}
  \item Input prompts: 4 base prompts (Doctors, Nurses, Engineers, Scientists).
  \item Image pairs per row: $N = 100$, generated with matched seeds.
  \item Generation model: Qwen-Image-2512.
  \item Embedding model: SigLIP~2 Giant.
\end{itemize}

\paragraph{Per-Row Results.} Table~\ref{tab:clip_dir} summarises the directional similarity scores. The \textit{Min} and \textit{Max} columns report the single-pair extremes across all 100 seeds, giving a sense of the score distribution, while \textit{Avg} is the primary metric (Equation~\ref{eq:clip_dir}).

\begin{table}[t]
\centering
\caption{\textbf{CLIP Directional Similarity --- per-profession summary
  ($N = 100$ pairs each).}}
\label{tab:clip_dir}
\begin{tabular}{lrrrrr}
\toprule
\textbf{Profession} & \textbf{Pairs} & \textbf{Min} & \textbf{Max}
  & \textbf{Avg} & \textbf{Std} \\
\midrule
Doctors    & 100 & $+0.029$ & $+0.105$ & $+0.0630$ & $\pm 0.017$ \\
Nurses     & 100 & $-0.046$ & $+0.125$ & $+0.0462$ & $\pm 0.031$ \\
Engineers  & 100 & $+0.048$ & $+0.164$ & $+0.0994$ & $\pm 0.021$ \\
Scientists & 100 & $+0.008$ & $+0.135$ & $+0.0803$ & $\pm 0.025$ \\
\midrule
Overall    & 400 & $-0.046$ & $+0.164$ & $+0.0708$ & $\pm 0.028$ \\
\bottomrule
\end{tabular}
\end{table}

\paragraph{Discussion.}
All four professions yield a positive overall average ($S_{\mathrm{dir}} = +0.071$), confirming that the enhanced prompts consistently move the image generator in the intended direction.

Engineers ($+0.099$) and Scientists ($+0.080$) show the strongest alignment. The enhanced prompts for these professions introduce many concrete, visually grounded elements (cultural garments, tools, lab coats with motifs) that the diffusion model can latch on to.

Doctors ($+0.063$) and Nurses ($+0.046$) score lower. The Nurses prompt is the only row containing negative per-pair scores ($\min = -0.046$), suggesting that some image pairs exhibit a visual drift orthogonal or even contrary to the textual direction. This is likely because subtle narrative elements in the enhanced prompt (\textit{e.g.}\ wristbands, passport photos) are harder for the model to render consistently.

Across all rows the score distribution is right-skewed with a long positive tail, indicating that while most seed pairs produce a modest positive alignment signal, a subset of seeds yields notably strong alignment --- reaching up to $+0.164$ for Engineers.

\subsection{Conclusion} \label{sec:conclusion}
The overall directional similarity of $+0.071$ provides quantitative evidence that RAG-enhanced prompts succeed in steering text-to-image generation toward more diverse and representative outputs. The metric is sensitive enough to distinguish between professions and prompt styles, making it a practical tool for iterative prompt refinement. Future work should examine whether higher scores also translate to improved perceptual diversity as rated by human evaluators.

\section{Optimal-Transport-Based Multimodal Alignment} \label{app:ot_alignment}

\subsection{Feature Extraction and Anisotropy Correction}

Let \(I\) denote an input image and \(T\) denote the corresponding text prompt. A Vision Transformer (ViT) is used to extract image-patch features, while a Text Transformer (e.g., CLIP) is used to extract token-level text features.

\subsubsection{Image processing flow}

The image \(I\) is divided into a sequence of \(M\) patches. After passing through the Vision Transformer and discarding the global classification token, we obtain raw local patch representations
\[
\tilde{V} \in \mathbb{R}^{M \times d},
\]
where \(d\) is the embedding dimension. Let \(\tilde{v}_i \in \mathbb{R}^{d}\) denote the raw representation of patch \(i\). To mitigate representation anisotropy, feature-wise Z-score standardization is applied:
\begin{equation}
v_i =
\frac{\tilde{v}_i - \mu_{\tilde{v}_i}}
{\sigma_{\tilde{v}_i} + \epsilon},
\qquad
\forall i \in \{1,\dots,M\},
\label{eq:image_zscore}
\end{equation}
where \(\mu_{\tilde{v}_i}\) and \(\sigma_{\tilde{v}_i}\) are the mean and standard deviation across the embedding dimensions of patch \(i\), and \(\epsilon = 10^{-5}\) is a small constant for numerical stability. The final set of image patch embeddings is
\[
V = \{v_i\}_{i=1}^{M}.
\]

\paragraph{Explanation.}
Equation~\eqref{eq:image_zscore} normalizes each patch embedding by centering it around zero and scaling it by its standard deviation. This reduces anisotropy in the visual embedding space and produces more stable feature representations for transport-based comparison.

\subsubsection{Text processing flow}

Similarly, the text prompt \(T\) is tokenized into a sequence. After passing through the Text Transformer, the Start-of-Sequence (SOS) and End-of-Text (EOT) tokens are discarded, yielding \(N\) valid token embeddings
\[
\tilde{U} \in \mathbb{R}^{N \times d}.
\]
Let \(\tilde{u}_j \in \mathbb{R}^{d}\) denote the raw representation of token \(j\). Applying the same Z-score standardization gives
\begin{equation}
u_j =
\frac{\tilde{u}_j - \mu_{\tilde{u}_j}}
{\sigma_{\tilde{u}_j} + \epsilon},
\qquad
\forall j \in \{1,\dots,N\},
\label{eq:text_zscore}
\end{equation}
and the final set of text token embeddings is
\[
U = \{u_j\}_{j=1}^{N}.
\]

\paragraph{Explanation.}
Equation~\eqref{eq:text_zscore} applies the same normalization strategy to the textual embeddings. Using the same stabilization for both modalities makes the subsequent alignment process more comparable and robust.

\subsection{Optimal Transport Alignment}

To measure the alignment between the visual patches \(V\) and textual tokens \(U\), uniform marginal distributions are defined over the image patches and text tokens:
\begin{equation}
p_i = \frac{1}{M}, \qquad \forall i,
\qquad\qquad
q_j = \frac{1}{N}, \qquad \forall j.
\label{eq:marginals}
\end{equation}

Let \(\Pi(p,q)\) denote the transport polytope containing all valid joint probability couplings \(\pi \in \mathbb{R}_{+}^{M \times N}\) satisfying the marginal constraints:
\begin{equation}
\Pi(p,q)=
\left\{
\pi \in \mathbb{R}_{+}^{M \times N}
\;\middle|\;
\sum_{j=1}^{N}\pi_{i,j}=p_i\ \forall i,
\quad
\sum_{i=1}^{M}\pi_{i,j}=q_j\ \forall j
\right\}.
\label{eq:transport_polytope}
\end{equation}

\paragraph{Explanation.}
Equation~\eqref{eq:marginals} assigns equal mass to each image patch and each text token. Equation~\eqref{eq:transport_polytope} defines the feasible set of transport plans that preserve these marginals, ensuring that the transport solution distributes probability mass consistently across both modalities.

\subsection{Wasserstein Distance (WD)}

The Wasserstein Distance evaluates direct cross-domain semantic alignment between visual and textual features. The semantic cost matrix \(M \in \mathbb{R}^{M \times N}\) is defined using cosine distance:
\begin{equation}
M_{i,j}
=
1-
\frac{\langle v_i,u_j\rangle}
{\|v_i\|_2\,\|u_j\|_2}.
\label{eq:semantic_cost}
\end{equation}

The optimal transport plan minimizes the total semantic cost:
\begin{equation}
W(V,U)
=
\min_{\pi \in \Pi(p,q)}
\sum_{i=1}^{M}\sum_{j=1}^{N}\pi_{i,j}M_{i,j}.
\label{eq:wd}
\end{equation}

\paragraph{Explanation.}
Equation~\eqref{eq:semantic_cost} measures how dissimilar each image patch is from each text token in the shared embedding space. Equation~\eqref{eq:wd} then finds the transport plan that minimizes the total cross-modal semantic mismatch. Therefore, a lower WD indicates better direct semantic agreement between the generated image content and the text prompt.

\subsection{Gromov--Wasserstein Distance (GWD)}

Unlike WD, the Gromov--Wasserstein Distance measures structural isomorphism by comparing intra-domain geometric relationships rather than direct cross-domain features. The intra-image and intra-text cost matrices are defined using normalized squared Euclidean distances:
\begin{equation}
C^{V}_{i,i'}
=
\frac{\|v_i-v_{i'}\|_2^2}
{\max_{k,l}\|v_k-v_l\|_2^2},
\qquad
C^{U}_{j,j'}
=
\frac{\|u_j-u_{j'}\|_2^2}
{\max_{k,l}\|u_k-u_l\|_2^2}.
\label{eq:intra_cost}
\end{equation}

Using the squared loss
\begin{equation}
L(a,b)=(a-b)^2,
\label{eq:gw_loss}
\end{equation}
the Gromov--Wasserstein objective becomes
\begin{equation}
GW(V,U)
=
\min_{\pi \in \Pi(p,q)}
\sum_{i=1}^{M}\sum_{i'=1}^{M}\sum_{j=1}^{N}\sum_{j'=1}^{N}
L\!\left(C^{V}_{i,i'},C^{U}_{j,j'}\right)\pi_{i,j}\pi_{i',j'}.
\label{eq:gwd}
\end{equation}

\paragraph{Explanation.}
Equation~\eqref{eq:intra_cost} captures pairwise relations within the image modality and within the text modality. Equation~\eqref{eq:gwd} compares these two relational structures through the transport plan. Lower GWD means that the internal organization of image features is more consistent with the internal organization of textual features, even when direct feature matching is not perfect.

\subsection{Fused Gromov--Wasserstein Distance (FGWD)}

The Fused Gromov--Wasserstein Distance jointly considers direct semantic matching and structural similarity. A hyperparameter \(\alpha \in [0,1]\) controls the trade-off between the Wasserstein term and the Gromov--Wasserstein term:
\begin{equation}
FGW_{\alpha}(V,U)
=
\min_{\pi \in \Pi(p,q)}
\sum_{i=1}^{M}\sum_{j=1}^{N}
\pi_{i,j}
\left[
(1-\alpha)M_{i,j}
+
\alpha
\sum_{i'=1}^{M}\sum_{j'=1}^{N}
L\!\left(C^{V}_{i,i'},C^{U}_{j,j'}\right)\pi_{i',j'}
\right].
\label{eq:fgwd}
\end{equation}

\paragraph{Explanation.}
Equation~\eqref{eq:fgwd} combines the direct semantic mismatch term from WD with the structural consistency term from GWD. When \(\alpha \to 0\), the objective approaches standard WD and prioritizes direct feature matching. When \(\alpha \to 1\), it approaches pure GWD and prioritizes structural alignment. Thus, FGWD provides a holistic multimodal alignment score that reflects both semantic fidelity and relational coherence.

\subsection{Iteration-wise Quantitative Results}

Table~\ref{tab:ot_metrics} summarizes the Optimal Transport-based metrics across five generative iterations. Since these metrics represent transport costs or distances, lower values indicate better alignment between the generated images and the corresponding text prompts.

\begin{table}[htbp]
\centering
\caption{Iteration-wise Optimal Transport alignment metrics. Lower values indicate better multimodal alignment.}
\label{tab:ot_metrics}
\begin{tabular}{lcccc}
\toprule
Iteration & Image Count & WD & GWD & FGWD \\
\midrule
\texttt{iteration\_1} & 4 & 0.748961 & 0.023537 & 0.404259 \\
\texttt{iteration\_2} & 5 & 0.754772 & 0.018855 & 0.405986 \\
\texttt{iteration\_3} & 2 & 0.763011 & 0.016279 & 0.408596 \\
\texttt{iteration\_4} & 5 & \textbf{0.745371} & 0.017281 & 0.399026 \\
\texttt{iteration\_5} & 4 & 0.746353 & \textbf{0.014763} & \textbf{0.398786} \\
\bottomrule
\end{tabular}
\end{table}

\paragraph{Explanation of the quantitative trends.}
The standard Wasserstein Distance (WD), defined in Equation~\eqref{eq:wd}, reflects direct semantic feature matching between image patches and text tokens. Its evolution is non-linear: WD increases from \texttt{iteration\_1} to \texttt{iteration\_3}, reaching its largest value at \texttt{iteration\_3}, and then decreases sharply in \texttt{iteration\_4} and remains low in \texttt{iteration\_5}. This indicates that direct semantic matching becomes temporarily more difficult in the middle iterations, but improves again in the later stages.

The Gromov--Wasserstein Distance (GWD), defined in Equation~\eqref{eq:gwd}, measures the compatibility of internal structural relationships between the visual and textual modalities. In contrast to WD, GWD shows an overall decreasing trend from \texttt{iteration\_1} to \texttt{iteration\_5}, with only a slight increase at \texttt{iteration\_4}. This suggests that the internal relational structure of the generated images becomes progressively more consistent with the structure of the text prompts over successive iterations.

The Fused Gromov--Wasserstein Distance (FGWD), defined in Equation~\eqref{eq:fgwd}, balances both semantic correspondence and structural similarity. Its trend broadly follows the behavior of WD, peaking at \texttt{iteration\_3} and improving afterward. The minimum FGWD is obtained at \texttt{iteration\_5}, indicating the strongest overall multimodal alignment when both direct feature agreement and structural coherence are taken into account.

\paragraph{Overall interpretation.}
Taken together, these results show that the iterative generation process improves multimodal alignment over time. Although the best pure semantic score (WD) is achieved at \texttt{iteration\_4}, the best structural score (GWD) and the best overall fused score (FGWD) are obtained at \texttt{iteration\_5}. Therefore, \texttt{iteration\_5} can be interpreted as the most balanced and robust final iteration when both semantic fidelity and structural consistency are considered jointly.

\section{Qualitative comparison} \label{app:qualitive_comparison}
In this section, we present results from the generated images for a prompt using a before-and-after prompt-enhancement setup. More images are provided in this link \footnote{\url{https://drive.google.com/drive/folders/1bP4u2bjupUGT9KMm9X7gHPLx4p1LLBsX}}

\begin{figure*}[t]
\centering




\begin{subfigure}[t]{0.49\textwidth}
    \centering
    \includegraphics[width=\linewidth]{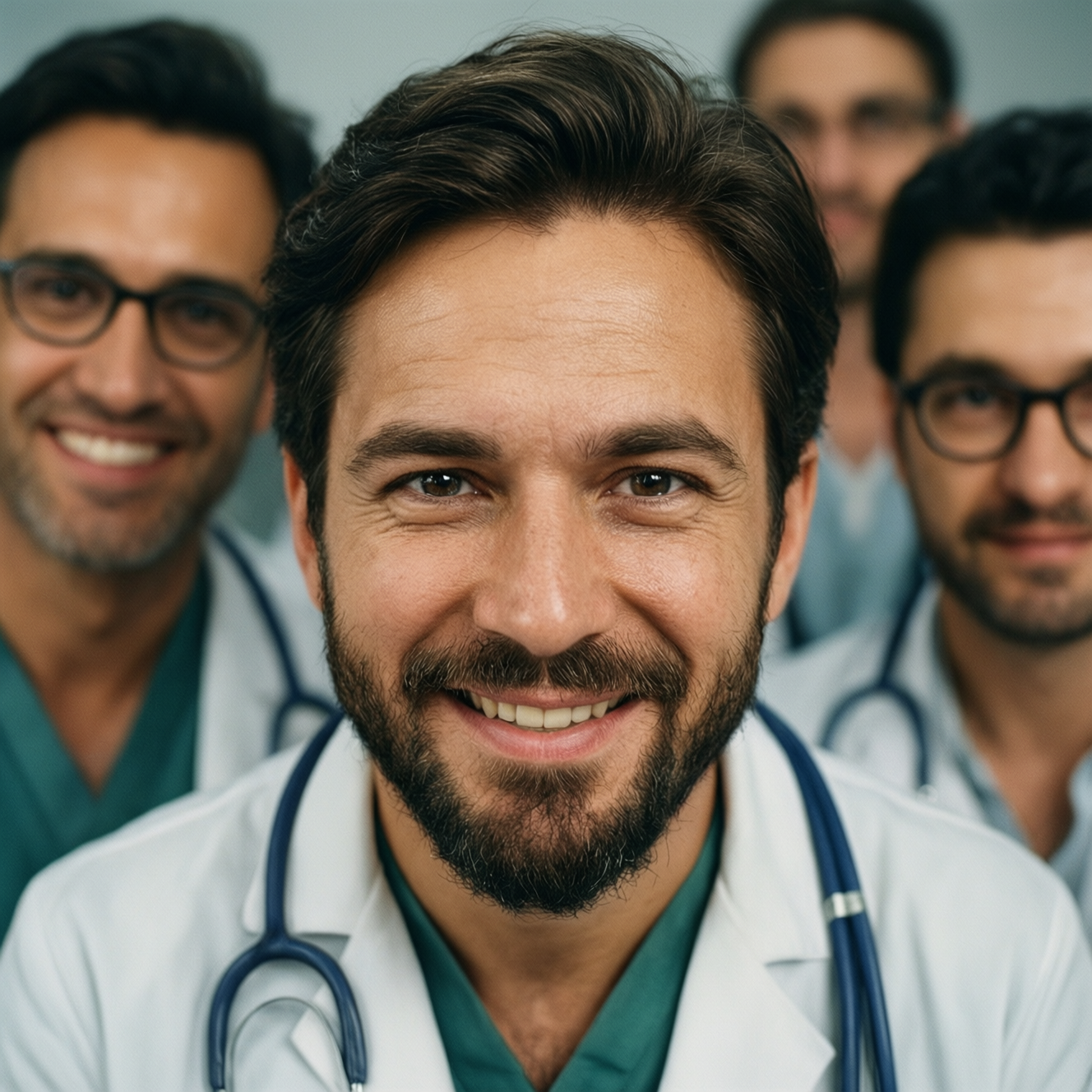}
\end{subfigure}
\hfill
\begin{subfigure}[t]{0.49\textwidth}
     \centering
     \includegraphics[width=\linewidth]{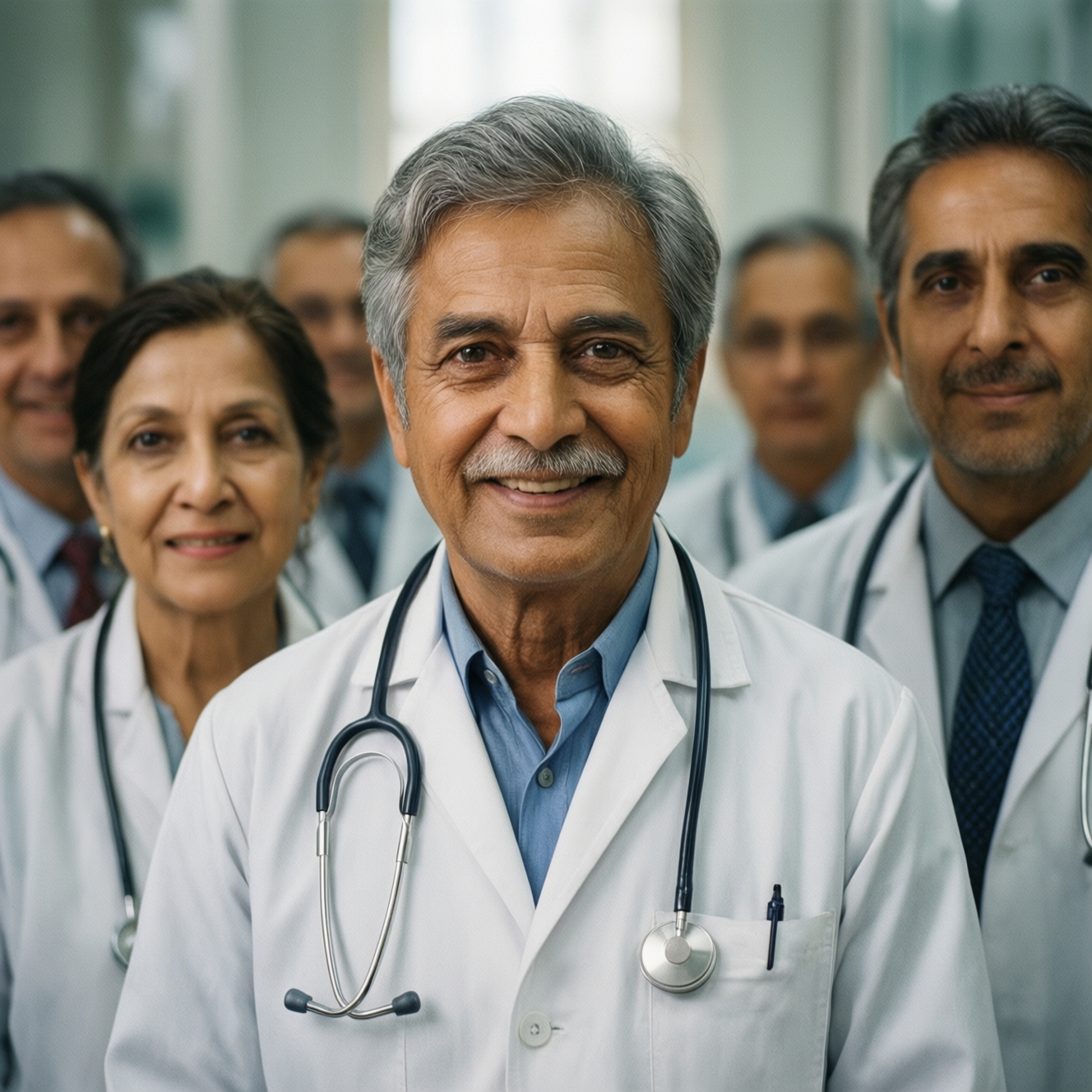}
\end{subfigure}

\caption{Qualitative comparison for Prompt Comparison 1. Left: an image generated with the generic prompt. Right: an image generated with the enhanced prompt.}
\label{fig:prompt_comparison_1}
\end{figure*}

\subsection{Prompt Comparison}
\noindent\textbf{Generic prompt.} Generate a close-up image of a group of Doctors who are smiling at my camera full face.

\medskip
\noindent\textbf{Enhanced prompt.} ["Generate a vibrant, close-up image showcasing the unity in diversity amongst a group of Doctors from various corners of the world. They should be arrayed in a way that celebrates not just professional camaraderie but a rich blend of cultural identities. Let's see youthful Indian interns buzzing with vigor, middle-aged Korean specialists attired in traditional Hanbok, expressing warmth with their radiant smiles, and sage Greek practitioners wearing the ceremonious attire associated with their graduation. Woven into this tapestry, capture also a dedicated Singaporean surgeon and a resilient American doctor who uses sign language, both effortlessly partaking in the group's jovial exchange. Include a charming visual dialogue that hints at shared interests, like some Doctors caught in a spirited debate about a popular North Korean talk show or others peering through vintage cameras, capturing the moment for posterity, a nod to their collective love for photography. Let the image represent various socioeconomic backgrounds and pay careful attention to individuals' expressions, authentically displaying their shared joy, commitment, and inherent strength. Each face, unique and expressive, counteracts stereotypes, their smiles not forced but natural, a testimony to their collective dedication and shared missionforced but natural, a testimony to their collective dedication and shared mission"]

\end{document}